\documentclass[letterpaper]{article}
\usepackage[preprint]{preprint}
\usepackage[hyphens]{url}
\usepackage{graphicx}
\usepackage{natbib}
\usepackage{caption}
\usepackage{algorithm}
\usepackage{algorithmic}
\usepackage{amsmath,amsfonts,amssymb}
\usepackage{amsthm}
\newtheorem{proposition}{Proposition}

\newtheorem{theorem}{Theorem}

\theoremstyle{definition}

\newtheorem{assumption}{Assumption}
\newtheorem{remark}{Remark}

\usepackage{multirow}
\usepackage{newfloat}
\usepackage{listings}
\DeclareCaptionStyle{ruled}{labelfont=normalfont,labelsep=colon,strut=off}
\floatstyle{ruled}
\newfloat{listing}{tb}{lst}{}
\floatname{listing}{Listing}

\usepackage{booktabs}

\title{When Muon Meets Task Interference: A Spectral Perspective on Continual Learning and Model Merging}
\author{
Shangge Liu\textsuperscript{\rm 1}, Yuehan Yin\textsuperscript{\rm 1}, Yinghuan Shi\textsuperscript{\rm 1}, Lei Wang\textsuperscript{\rm 2}, Wenbin Li\textsuperscript{\rm 1}\corresponding\\
}
\affiliations{
\textsuperscript{\rm 1}State Key Laboratory for Novel Software Technology, Nanjing University\\
\textsuperscript{\rm 2}University of Wollongong\\
}

\begin{document}

\maketitle

\begin{abstract}
Continual learning (CL) and model merging (MM) both aim to obtain a single model that performs well across multiple tasks, challenged respectively by catastrophic forgetting and weight-disentanglement error. In the literature, these difficulties are merely treated separately and mitigated through a variety of solutions, while the geometry induced by the base optimizer is treated as an implementation detail. In this work, we show that the two difficulties are in fact two instances of the same phenomenon: a parameter update useful for one task shifts the model's outputs on another. We formalize this shared phenomenon as \textit{task interference} and reduce it to a common layer-wise Frobenius inner product $\langle \Delta W_\ell, J_\ell(x)\rangle_F$. This quantity, in turn, is utilized to expose the role of the optimizer. We theoretically derive an upper bound that isolates the spectral norm $\|\Delta W_\ell\|_2$ as an optimizer-controllable factor of task interference, and a per-mode analysis shows that this bound tracks the dominant part of the empirical interference. Specifically, we then identify the recent Muon optimizer as a mechanism that regulates this factor by construction. Our work reveals that its elegant control on spectral norm tightens the interference bound for both CL and MM, positioning Muon as a principled optimizer-centric approach complementary to existing solutions. Our theoretcal analysis is well validated by experimental results. Replacing the AdamW optimizer with Muon improves accuracy by up to +5.02 points on the eight-task model-merging benchmark across three CLIP backbones. For continual learning, Muon also delivers uniformly positive gains across ten class-incremental protocols, three task-incremental protocols, and the 11-task MTIL benchmark.

\end{abstract}
\section{Introduction}
Modern pre-trained models are expected to acquire capabilities from multiple tasks, domains, and data streams while retaining previously learned knowledge. Two major paradigms address this objective under different practical constraints. Continual learning (CL)~\cite{Li2018LwF,li2025libcontinual} updates a single model sequentially, often with limited or no access to earlier data. And it's challenged by \textit{catastrophic forgetting}~\cite{catastrophicmccloskey1989}: learning a new task could degrade performance on previous ones. Model merging (MM)~\cite{ZSCL}, by contrast, fine-tunes task-specific models independently and combines them in weight space into a single model, but the resulting task vectors may interfere after composition and violate \textit{weight disentanglement}~\cite{ortizjimenez2023}. 

\begin{figure}[t]
\centering
\includegraphics[width=\linewidth]{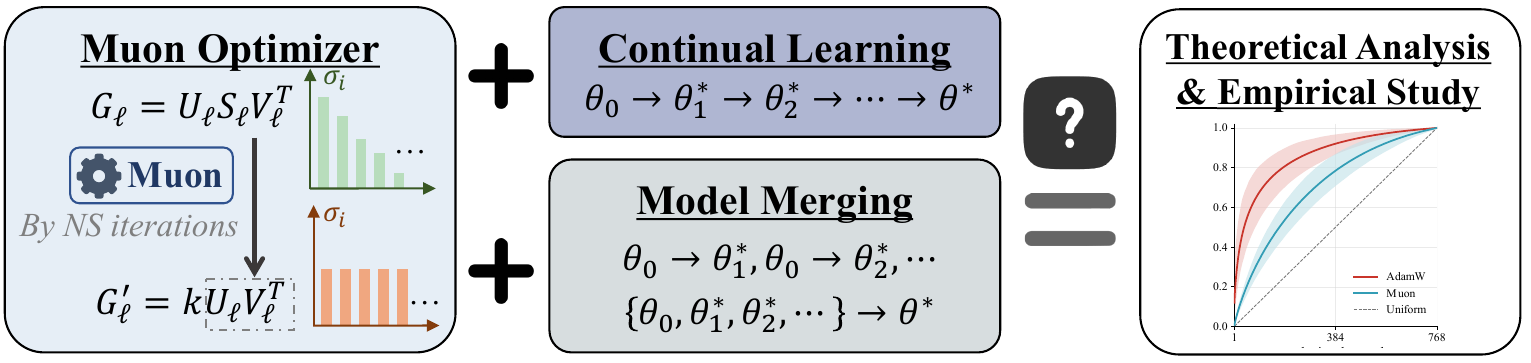}
\caption{Overview of this study. We reinterpret the Muon optimizer through the lens of continual learning and model merging, and provide a unified theoretical analysis together with empirical validation on both paradigms.}
\label{fig:framework}
\vspace{-15pt}
\end{figure}

Existing methods mitigate these difficulties through several mechanisms. For CL, existing solutions add regularizations on important weights~\cite{EWC}, replay stored or generated samples~\cite{iCaRL}, project gradients away from previous task subspaces~\cite{farajtabar2020}, or allocate task-specific parameters~\cite{rusu2016} to mitigate forgetting. For MM, methods resolve sign conflicts~\cite{yadav2023ties}, learn merging coefficients~\cite{yang2024adamerging}, or regularize weight update~\cite{liu2026understanding} to promote weight disentanglement.
These methods operate at the level of data access, objectives, architectures, parameter spaces, or merging rules. In contrast, the geometry induced by the base optimizer is usually treated as an implementation detail, without sufficient attention. This leaves a fundamental question insufficiently answered: \textit{how does the optimizer affect forgetting in CL and disentanglement error in MM, and is there a principled way to analyse the impact of the optimizer alone?}

To this end, we observe that although catastrophic forgetting and weight-disentanglement error are formalized separately, they are in fact two instances of the same phenomenon:\textit{ \textbf{a parameter update that is useful for one task can change the model's outputs on another}}. We refer to this shared phenomenon as \textit{\textbf{task interference}}, and formalize it uniformly as follows. In both cases, task interference takes the form of an output change relative to a fixed reference point, and reduces to a common layer-wise quantity. Concretely, this common quantity is proved the Frobenius inner product $\langle \Delta W_\ell,\, J_\ell(x)\rangle_F$ between the weight update and the Jacobian. 
This unification is the starting point of our analysis. As will be seen, any mechanism that provably shrinks this inner product mitigates forgetting in CL and reduces disentanglement error in MM at once.

We then decompose this shared quantity into two distinct factors: one determined by the pre-trained backbone and the task distribution, and the other being the spectral norm of the update $\|\Delta W_\ell\|_2$, which is shaped entirely by training and hence the optimizer.
This decomposition isolates the spectral norm as an optimizer-controllable factor of task interference. 

To investigate how much the spectral norm governs task interference in practice, we conduct a per-mode analysis, which shows 

an optimizer that keeps the spectral norm small could reduce interference in both paradigms.

Upon the above findings, we point out that the recent Muon optimizer~\cite{jordan2024muon} is precisely such a mechanism. It has been shown to be the closed-form solution of steepest descent under the spectral norm~\cite{bernstein2024}, which is exactly the norm that our decomposition isolates as the optimizer-controllable factor of task interference.
We therefore reinterpret Muon through the lens of task interference. That is, it regulates exactly the quantity that governs task interference, and it does so by construction, without touching the loss, the parameterization, or the data pipeline.
A single spectral-norm control thus simultaneously tightens the interference bound for both continual learning and model merging, positioning Muon as a principled, optimizer-centric approach complementary to the existing mechanism-centric solutions in both literatures.

We experimentally validate our theoretical analysis across both of CL and MM paradigms with all other components fixed.
On the standard eight-task model-merging benchmark, replacing the AdamW optimizer with Muon during fine-tuning improves absolute accuracy by $+5.02$, $+3.42$, and $+3.82$ points on CLIP-pretrained ViT-B/32, ViT-B/16, and ViT-L/14. For CL, Muon also delivers uniformly positive gains across ten class-incremental protocols and three task-incremental protocols, and improves all three metrics on the 11-task MTIL benchmark. 

Our contributions in this work are as follows.
\begin{itemize}
    \item We unify catastrophic forgetting in CL and weight-disentanglement error in MM as two instances of \textit{task interference}, and show that both reduce to the same term $\langle \Delta W_\ell, J_\ell(x)\rangle_F$, providing a common analytical target for two paradigms that are typically studied separately.
    \item Our work derives an upper bound on task interference that isolates $\|\Delta W_\ell\|_2$ as an optimizer-controllable factor, and show via a per-mode analysis that the bound tracks the dominant part of the empirical interference.
    \item We connect the interference bound to Muon's spectral-norm steepest descent, showing that a single change of optimizer mitigates forgetting in CL and disentanglement error in MM, positioning Muon as a principled optimizer-centric approach complementary to existing  solutions.
    \item Our experimental study validates our analysis on eight-task model merging across three CLIP backbones and on CIL, TIL, and MTIL continual-learning benchmarks, directly confirming the revealed mechanism.
\end{itemize}

\section{Related Work}

\paragraph{Continual Learning.}
Continual learning aims to incrementally acquire knowledge from a sequence of tasks while avoiding catastrophic forgetting. Existing methods can be broadly grouped into five categories~\cite{wang2024survey,li2025libcontinual}: regularization-based methods that penalize updates to important weights~\cite{EWC}; replay-based methods that store or generate past samples~\cite{iCaRL}; optimization-based methods that project gradients onto subspaces orthogonal to old tasks~\cite{farajtabar2020}; representation-based methods that leverage stable feature representations from self-supervised or pre-trained backbones~\cite{madaan2021}; and architecture-based methods that allocate task-specific sub-networks~\cite{mallya2018}. Our work falls within the optimization-based category, but departs from prior work in that we do not introduce any task-specific projection. Interference is controlled purely by the geometry of the base optimizer.

\paragraph{Model Merging.}
Model merging combines independently fine-tuned models into a unified one directly in weight space, without additional training or data access~\cite{ilharco2022editing}. Existing solutions fall into during-merging methods, which design combination algorithms applied to already-trained checkpoints~\cite{yadav2023ties,yang2024adamerging}, and pre-merging methods, which modify fine-tuning to produce more mergeable models~\cite{ortizjimenez2023,liu2026understanding}. Our work belongs to the pre-merging category, but differs from prior work in that the improvement comes entirely from the base optimizer.

\paragraph{Muon Optimizer.}
Muon~\cite{jordan2024muon} is a recently proposed optimizer for the hidden-layer weight matrices of neural networks. At each step it orthogonalizes the gradient before applying the update, encouraging balanced gradient flow across the singular value spectrum of each weight matrix. Muon has demonstrated competitive convergence relative to AdamW in large-scale pretraining, and recent work further scales it to billion-parameter language models~\cite{liu2025muon,team2025kimi} and analyzes its theoretical properties~\cite{chen2025muon}. 
Most prior work studies Muon through the lens of optimization efficiency in single-task pretraining. Muon-OGD~\cite{lu2026muonogd}, combines Muon with orthogonal gradient projection to mitigate forgetting in LLM continual learning, but relies on an additional projection mechanism rather than the base optimizer alone, and does not consider model merging.
In this work, we take a different perspective and show that the geometry Muon imposes on updates directly bounds task interference in both continual learning and model merging. 

A more comprehensive discussion of related work is provided in Appendix.

\begin{figure*}[t]
\centering
\includegraphics[width=\linewidth]{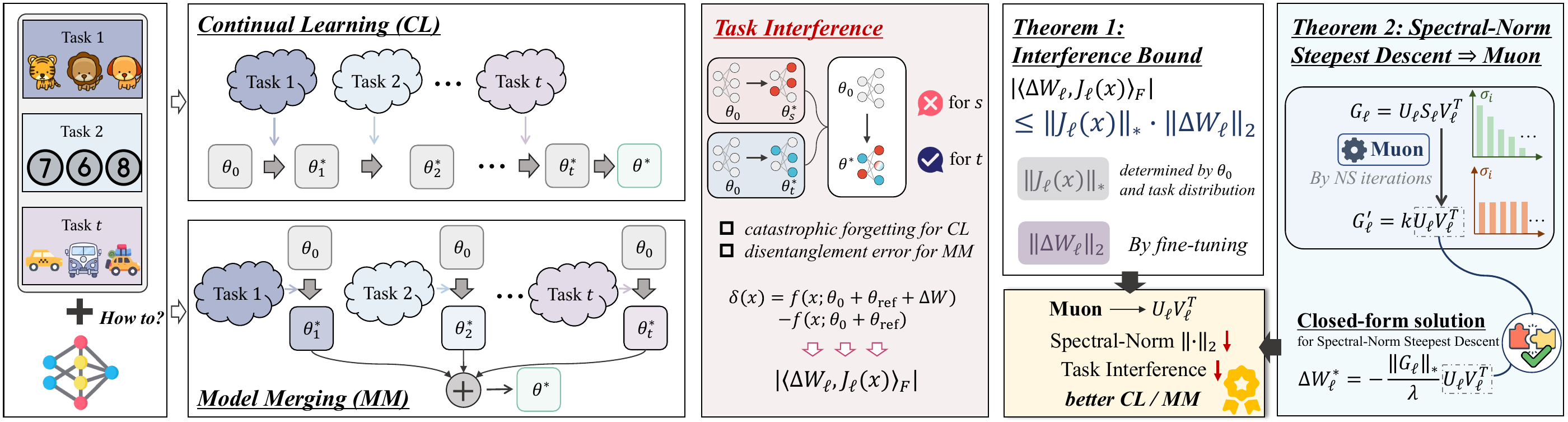}
\caption{A unified spectral framework for CL and MM. Left: Both paradigms target a shared multi-task optimum $\theta^\star$ through different protocols. Middle: A first-order Taylor expansion unifies catastrophic forgetting (CL) and disentanglement error (MM) as the same layer-wise term $\langle \Delta W_\ell, J_\ell(x)\rangle_F$, which Theorem~\ref{thm:bound} bounds by $\|J_\ell(x)\|_*\cdot\|\Delta W_\ell\|_2$ and isolates the spectral norm as an optimizer-controllable factor. Right: Theorem~\ref{thm:spectral_steepest} identifies Muon as the closed-form spectral-norm steepest descent, which uniformizes the update spectrum and thus tightens the interference bound for both paradigms.}
\label{fig:framework-full}
\vspace{-10pt}
\end{figure*}
\section{The Proposed Unified Framework}

Figure~\ref{fig:framework-full} summarizes the unified framework. We first cast CL and MM under a common objective (left), unify their difficulties as task interference (middle), derive an upper bound that isolates spectral-norm as an optimizer-controllable factor (Theorem~\ref{thm:bound}, middle-right), and finally identify Muon as the closed-form solution to spectral-norm steepest descent (Theorem~\ref{thm:spectral_steepest}, right).

\subsection{A Unified View of CL and MM}\label{subsec:unified}

\textbf{Basic Setup.} Consider a neural network $f(x;\theta)$ with parameters $\theta \in \mathbb{R}^P$ and pre-trained weights $\theta_0$. Given a sequence of $T$ tasks $\{1,\ldots,t,\ldots,T\}$ with data domains $\{\mathcal{D}_t\}$ and losses $\mathcal{L}_t(\theta) := \mathbb{E}_{(x,y)\sim\mathcal{D}_t}[\ell(f(x;\theta),y)]$, our goal is to obtain a single set of parameters that performs well on all tasks,
\begin{equation}\small
    \theta^\star \;\in\; \arg\min_{\theta}\ \sum_{t=1}^{T} \mathcal{L}_t(\theta).\label{eq:objective}
\end{equation}
Continual learning and model merging are two paradigms that approach Eq.~\eqref{eq:objective} under different practical constraints.

\textbf{Continual learning (CL)} processes the tasks sequentially, producing a chain $\theta_0 \to \theta_1^\star \to \cdots \to \theta_T^\star$ where $\theta_t^\star$ is obtained by fine-tuning from the previous checkpoint $\theta_{t-1}^\star$ under restricted access to earlier data~\cite{li2025libcontinual}. The final iterate $\theta_T^\star$ is taken as the approximation to $\theta^\star$.

\textbf{Model merging (MM)} fine-tunes each task independently from $\theta_0$ to obtain $\{\theta_t^\star\}_{t=1}^T$, and constructs a multi-task model by composing task vectors $\tau_t := \theta_t^\star - \theta_0$~\cite{ilharco2022editing}:
\begin{equation}\small
    \theta_{\text{MM}} = \theta_0 + \sum_{t=1}^T \alpha_t \tau_t,
\end{equation}
which is taken as the approximation to $\theta^\star$.

\textbf{The central challenge} in each paradigm is that subsequent updates shift the model's outputs on other tasks' domains, preventing convergence to Eq.~\eqref{eq:objective}. 

In CL, this shift accumulates across sequential fine-tuning steps and is known as \textit{catastrophic forgetting}~\cite{catastrophicmccloskey1989}. Following the formalization of~\citet{Doan2021}, after fine-tuning on task $t$, the drift on the previous task's domain $\mathcal{D}_{t-1}$ is
\begin{equation}\small
    \delta_{\text{CF}}(x) = f(x;\theta_{t-1}^\star + \Delta W) - f(x;\theta_{t-1}^\star),\ x\in\mathcal{D}_{t-1},\label{eq:cf}
\end{equation}
where $\Delta W := \theta_t^\star - \theta_{t-1}^\star$ is the update produced by fine-tuning on task $t$.

In MM, the shift arises from composing task vectors. Ideal composition would leave $\mathcal{D}_s$ unaffected by any $\tau_t\ (t\ne s)$, a property known as \textit{weight disentanglement}~\cite{ortizjimenez2023}. And the shift is precisely the departure from this property. For a task $t\ne s$, the disentanglement error on $\mathcal{D}_s$ is
\begin{equation}\small
\delta_{\text{DE}}(x) = f(x;\theta_0+\alpha_s\tau_s+\Delta W) - f(x;\theta_0+\alpha_s\tau_s),\ x\in\mathcal{D}_s,\label{eq:wd}
\end{equation}
where $\Delta W := \alpha_t\tau_t$ is the scaled task vector of task $t$.

Eqs.~\eqref{eq:cf} and~\eqref{eq:wd} both take the form $f(x;\theta_{\text{ref}}+\Delta W)-f(x;\theta_{\text{ref}})$ on a source-task domain, differing only in what $\theta_{\text{ref}}$ and $\Delta W$ represent. 
Under the NTK linearization hypothesis~\cite{NTK}, we expand $f(x;\theta)$ to first order in $\theta$ around $\theta=\theta_{\text{ref}}$. Substituting this expansion into $\delta(x)=f(x;\theta_{\text{ref}}+\Delta W)-f(x;\theta_{\text{ref}})$ and decomposing layer-wise yields
\begin{equation}\small
    \delta(x) \;\approx\; \sum_\ell \bigl\langle \Delta W_\ell,\, J_\ell(x)\bigr\rangle_F,
    \label{eq:unified}
\end{equation}
where $J_\ell(x) := \nabla_{W_\ell} f(x;\theta_{\text{ref}})$ is the layer-$\ell$ Jacobian and $\delta \in \{\delta_{\text{CF}}, \delta_{\text{DE}}\}$.

Eq.~\eqref{eq:unified} is the analytical target of this paper. $\delta_{\text{CF}}$ and $\delta_{\text{DE}}$ are governed by the same term, which we call the layer-wise \emph{task interference} $\langle \Delta W_\ell, J_\ell(x)\rangle_F$. Any mechanism that provably shrinks this term therefore mitigates forgetting and promotes disentanglement simultaneously. Building on this unification, we first derive a spectral-norm upper bound on task interference, and then show that the Muon optimizer controls this bound by construction.

\subsection{A Spectral-Norm Bound on Task Interference}\label{subsec:bound}
The layer-wise task interference $\langle \Delta W_\ell, J_\ell(x)\rangle_F$ governs both catastrophic forgetting and disentanglement error. We now derive an analytical upper bound on it.
\begin{theorem}[Layer-wise Interference Bound]
\label{thm:bound}
For any $\Delta W_\ell, J_\ell(x) \in \mathbb{R}^{m_\ell \times n_\ell}$,
\begin{equation}\small
\bigl|\langle \Delta W_\ell,\, J_\ell(x) \rangle_F\bigr|\;\le\;\|\Delta W_\ell\|_2 \cdot \|J_\ell(x)\|_*,
\label{eq:bound}
\end{equation}
where $\|\cdot\|_2$ denotes the spectral norm of a matrix and $\|\cdot\|_*$ the nuclear norm of a matrix.
\end{theorem}

\begin{proof}[Proof sketch]
Writing $J_\ell(x) = \sum_i \sigma_i u_i v_i^\top$ and expanding the Frobenius
inner product gives $\langle \Delta W_\ell, J_\ell(x)\rangle_F= \sum_i \sigma_i \cdot u_i^\top \Delta W_\ell v_i$. Since $u_i$ and $v_i$ are unit vectors, the operator definition of the spectral norm gives $|u_i^\top \Delta W_\ell v_i| \le \|\Delta W_\ell\|_2$ for each $i$. The triangle inequality and $\|J_\ell(x)\|_* = \sum_i \sigma_i$ then yield Eq.~\eqref{eq:bound}. See Appendix for details.
\end{proof}

\begin{figure}[t]
\centering
\includegraphics[width=0.9\linewidth]{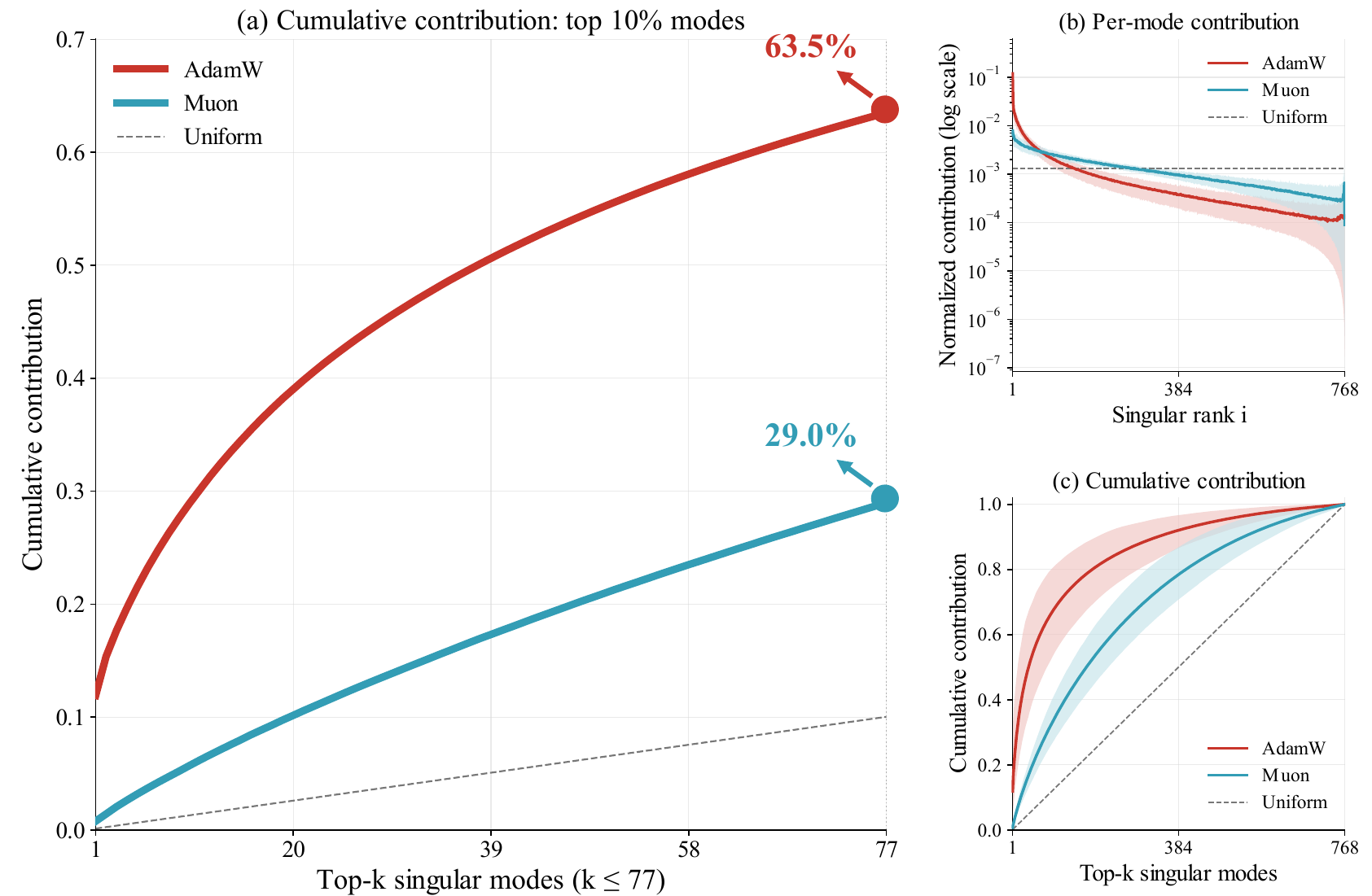}
\caption{Empirical validation of Remark~\ref{rem:tightness} on ViT-B/32. We plot the per-mode contribution of the SVD modes of $\Delta W_\ell$. (a) Cumulative contribution of the top-10\% singular modes; (b) per-mode contribution on a log scale; (c) full cumulative curve across all modes. Under AdamW, the top-10\% of modes account for $63.5\%$ of the total interference, confirming that interference is sharply concentrated on a few dominant modes. Under Muon, the same fraction is $29.0\%$ and the curve is nearly flat. This validates Remark~\ref{rem:tightness}: shrinking $\|\Delta W_\ell\|_2$ suppresses precisely the modes that dominate the true interference, so the bound of Theorem 1 tracks the actual quantity rather than merely upper-bounding it.}
\label{fig:contri}
\vspace{-15pt}
\end{figure}

Eq.~\eqref{eq:bound} decomposes task interference into two factors: one
determined by the model and the data, and the other by the optimizer. 

Strictly, $J_\ell(x)$ is evaluated at the reference point $\theta_{\text{ref}}$, which is $\theta_{t-1}^\star$ in CL and $\theta_0 + \alpha_s\tau_s$ in MM. 

However, in both paradigms fine-tuning usually keeps $\theta_{\text{ref}}$ within a bounded neighborhood of $\theta_0$ of radius $R:=\|\theta_{\text{ref}}-\theta_0\|_2$. Under a standard local-Lipschitz assumption on the Jacobian with constant $L_{J,\ell}$, Proposition~1 in Appendix gives 
\begin{equation}\small
    \mathbb E_{x\sim\mathcal D_s}\big[\|J_\ell(x;\theta_{\text{ref}})\|_*\big]\;\le\;\mathcal J_\ell(\mathcal D_s;\theta_0)\;+\;L_{J,\ell}\,R,
\end{equation}
where $\mathcal J_\ell(\mathcal D_s;\theta_0):=\mathbb E_{x\sim\mathcal D_s}\!\big[\|J_\ell(x;\theta_0)\|_*\big]$ depends only on the pre-trained backbone and the source-task distribution, $L_{J,\ell}$ is the local Lipschitz constant of the layer-$\ell$ Jacobian at $\theta_0$, and $R$ is the fine-tuning radius. None of these quantities is affected by the choice of optimizer.
The spectral norm $\|\Delta W_\ell\|_2$, in contrast, is a product of training. In both CL and MM, it is directly shaped by the choice of optimizer.

\begin{figure}[t]
\centering
\includegraphics[width=\linewidth]{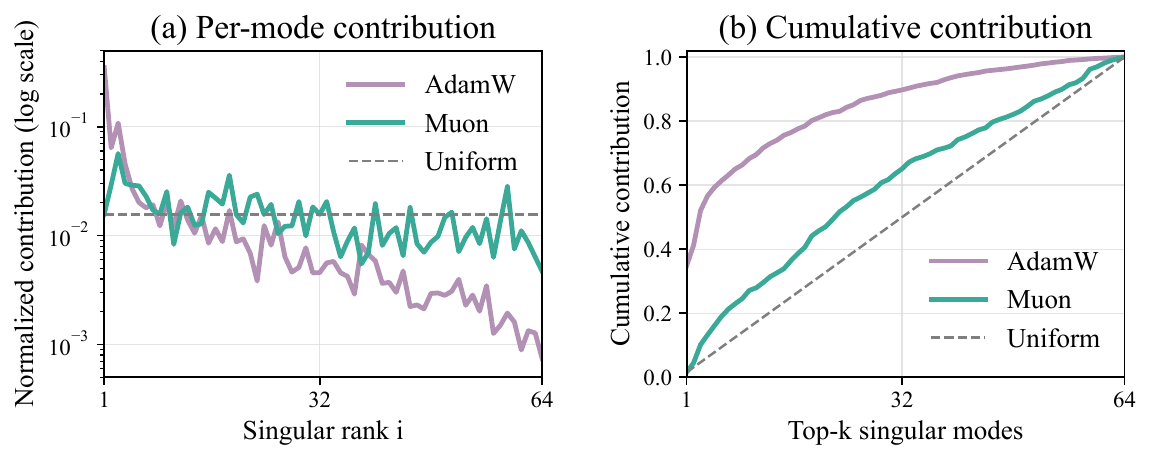}
\caption{Per-mode contribution of the MoE-Adapter updates $\Delta W$ along the MTIL sequence. (a) $\mathrm{Contrib}_i$ vs. singular rank $i$ on a log scale. (b) Cumulative contribution over the top-$k$ modes. The CL counterpart of Figure~\ref{fig:contri}.}
\label{fig:contri-cl}
\vspace{-15pt}
\end{figure}

\begin{remark}[Empirical relevance of the bound]
\label{rem:tightness}
Theorem~\ref{thm:bound} bounds the interference, but the bound itself does not indicate how much of the actual interference is captured by the modes that a small $\|\Delta W_\ell\|_2$ would constrain. To check, write the SVD of the update as $\Delta W_\ell = \sum_i s_i p_i q_i^\top$ with $s_1 \geq s_2 \geq \cdots \geq 0$, and define the per-mode contribution $\mathrm{Contrib}_i(x) \;:=\; s_i \cdot \bigl|\langle p_i q_i^\top,\, J_\ell(x)\rangle_F\bigr|$, and it is easy to prove that $|\langle \Delta W_\ell, J_\ell(x)\rangle_F| \leq \sum_i \mathrm{Contrib}_i(x)$, with equality achieved when the signed inner products share a common sign across modes. Empirically, as shown in Figures~\ref{fig:contri} and~\ref{fig:contri-cl}, for task vectors trained by AdamW, $\mathrm{Contrib}_i(x)$ concentrates sharply on the top singular modes of $\Delta W_\ell$. The top-$k$ cumulative contribution $C_k(x) := \sum_{i=1}^{k} \mathrm{Contrib}_i(x)$ grows rapidly in $k$. On ViT-B/32, the top 10\% of singular modes already account for 63.5\% of the total $\sum_i \mathrm{Contrib}_i(x)$ (Figure~\ref{fig:contri}(a)). 
An update whose singular value spectrum is flatter at a matched Frobenius scale $\|\Delta W_\ell\|_F$ places less energy on these leading modes, and this would usually yield a smaller $\sum_i \mathrm{Contrib}_i(x)$.
And Figure~\ref{fig:relative_interference} shows that replacing AdamW with Muon consistently reduces the empirical relative interference across all source tasks. Theorem~\ref{thm:bound} therefore identifies an empirically meaningful target rather than a loose upper bound.
\end{remark}

\subsection{Muon Optimizer Controls the Interference Bound}\label{subsec:muon}

Theorem~\ref{thm:bound} identifies $\|\Delta W_\ell\|_2$ as the optimizer-controllable factor in task interference. We now show that the Muon optimizer regulates this quantity by construction, by realizing steepest descent under the spectral norm.

\begin{theorem}[Spectral-norm steepest descent~\cite{bernstein2024}]
\label{thm:spectral_steepest}
Let $G_\ell \in \mathbb{R}^{d_\ell^{\mathrm{out}} \times d_\ell^{\mathrm{in}}}$ have compact SVD $G_\ell = U_\ell \Sigma_\ell V_\ell^\top$ with $r = \mathrm{rank}(G_\ell)$. For any sharpness $\lambda > 0$, the spectral-norm steepest-descent problem
\begin{equation}\small
    \Delta W_\ell^\star = \arg\min_{\Delta W_\ell}\left[\,\langle G_\ell,\, \Delta W_\ell \rangle_F+ \tfrac{\lambda}{2}\|\Delta W_\ell\|_2^2\,\right]
\end{equation}
admits the closed-form solution
\begin{equation}\small
    \Delta W_\ell^\star\;=\;-\,\frac{\|G_\ell\|_*}{\lambda}\,\cdot\, U_\ell V_\ell^\top.
    \label{eq:spectral_steepest}
\end{equation}
\end{theorem}
\begin{proof}
See Appendix.
\end{proof}

\begin{table*}[t]\small
\centering
\setlength{\tabcolsep}{4.5pt}
\begin{tabular}{lcccccc}
\toprule
& \multicolumn{2}{c}{ViT-B/32, 8 tasks}
& \multicolumn{2}{c}{ViT-B/16, 8 tasks}
& \multicolumn{2}{c}{ViT-L/14, 8 tasks} \\
\cmidrule(lr){2-3}\cmidrule(lr){4-5}\cmidrule(lr){6-7}
Method & Abs.Acc. $\uparrow$ & Norm.Acc. $\uparrow$ & Abs.Acc. $\uparrow$ & Norm.Acc. $\uparrow$ & Abs.Acc. $\uparrow$ & Norm.Acc. $\uparrow$ \\
\midrule
Zero-shot
    & 47.74 & -- & 54.22 & -- & 64.54 & -- \\
\midrule
Non-linear FT + AdamW  & 70.32 & 77.56 & 75.39 & 75.39 & 84.07 & 89.19 \\
Non-linear FT + Muon   & 75.34 & 82.79 & 78.81 & 84.64 & 87.89 & 93.26 \\
\quad \textit{Gain}     & \textbf{+5.02} & \textbf{+5.23} & \textbf{+3.42} & \textbf{+9.25} & \textbf{+3.82} & \textbf{+4.07} \\
\midrule
TTA + AdamW & 74.68 & 85.27 & 78.97 & 87.48 & 86.19 & 93.14 \\
TTA + Muon  & 77.41 & 87.39 & 80.72    & 87.98    & 88.03    & 93.78    \\
\quad \textit{Gain}     & \textbf{+2.73} & \textbf{+2.12} & \textbf{+1.75}    & \textbf{+0.50}    & \textbf{+1.84}    & \textbf{+0.64}    \\
\midrule
OrthoReg + AdamW & 73.41 & 93.93 & 77.68 & 93.62 & 88.23 & 100.08 \\
OrthoReg + Muon  & 78.05 & 102.87 & 82.21 & 102.10 & 89.83 & 107.35 \\
\quad \textit{Gain}     & \textbf{+4.64} & \textbf{+8.94} & \textbf{+4.53} & \textbf{+8.48} & \textbf{+1.60} & \textbf{+7.27} \\
\bottomrule
\end{tabular}
\vspace{-5pt}
\caption{Eight-task addition on CLIP backbones. Absolute Accuracy (Abs.Acc.) and Normalized Accuracy (Norm.Acc.) are reported. ``+ Muon'' denotes optimizing all matrix-valued parameters with Muon and the remaining parameters with auxiliary AdamW. Across all three backbones, replacing AdamW with Muon improves both metrics under every pre-merging strategy.}\label{tab:mm_main}
\vspace{-10pt}
\end{table*}

The update direction $U_\ell V_\ell^\top$ points along the singular subspaces of $G_\ell$, with all singular values normalized to one~\cite{bernstein2024}.
The Muon optimizer realizes the spectral-norm steepest descent of Theorem~\ref{thm:spectral_steepest} by a purely numerical route. At each step it smooths the gradient with Nesterov momentum, Frobenius-normalizes the result, and applies $K$ Newton--Schulz iterations, whose terminal iterate $X_K$ approximates the polar factor $U_\ell V_\ell^\top$ (see Appendix and~\citet{jordan2024muon} for details). The resulting update $\Delta W_\ell \;=\; -\eta \cdot X_K$ inherits the spectral geometry of $U_\ell V_\ell^\top$: its singular values are uniformly equal to $\eta$, regardless of how the singular energy of $G_\ell$ is distributed. 

Substituting this into Theorem~\ref{thm:bound} suppresses the interference bound at every layer, uniformly across the CL update in Eq.~\eqref{eq:cf} and the scaled task-vector contribution in Eq.~\eqref{eq:wd}. 
Because Muon's update direction $U_\ell V_\ell^\top$ matches the spectral-norm steepest-descent direction of Theorem~\ref{thm:spectral_steepest}, Muon descends in precisely the norm $\|\cdot\|_2$ that governs the Theorem~\ref{thm:bound} bound. As an optimizer, it therefore tightens this bound by construction, since the update geometry it enforces is exactly the geometry that the bound measures.
We conclude that Muon is inherently better suited to both continual learning and model merging than optimizers that leave the update spectrum unconstrained. A single spectral-norm control simultaneously mitigates catastrophic forgetting in CL (Eq.~\eqref{eq:cf}) and reduces disentanglement error in MM (Eq.~\eqref{eq:wd}).
\section{Experiments on Model Merging}\label{sec:exp_model_merging}

\begin{figure}[t]
\centering
\includegraphics[width=\linewidth]{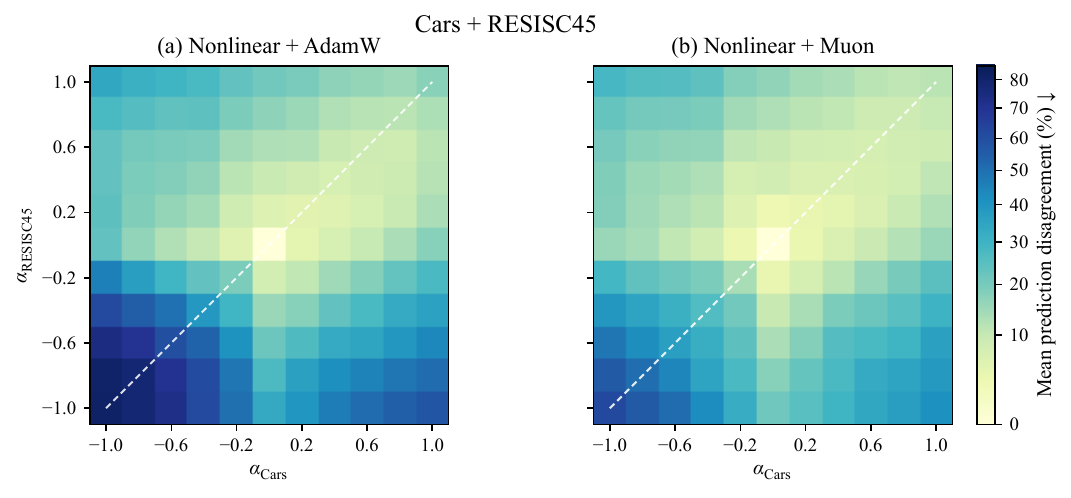}
\caption{ Weight-disentanglement heatmap for Cars + RESISC45. Each cell reports the prediction disagreement (\%) and lower is better. (a) AdamW: low-disagreement region narrowly confined near the origin. (b) Muon: the same region expands in both axes. This matches the reduction in interfence predicted by Theorem~\ref{thm:bound}.}
\label{fig:wd}
\vspace{-15pt}
\end{figure}

\subsection{Experimental Setup}

\textbf{Datasets and backbones.}
We follow the standard eight-task addition protocol of Model Merging~\cite{ortizjimenez2023,liu2026understanding} on the CLIP-pretrained ViT-B/32, ViT-B/16, and ViT-L/14 backbones, respectively.

\noindent\textbf{Muon implementation.}
Muon is applied to the image encoder's matrix parameters, with an auxiliary AdamW handling the remaining parameters, following the standard protocol~\citep{jordan2024muon}.

\noindent\textbf{Baselines.}
We compare Muon against three widely used pre-merging strategies, each paired with both AdamW and Muon (denoted by the ``$+$ Muon'' suffix): Non-linear Fine-tuning~\cite{ilharco2022editing}, Tangent Task Arithmetic (TTA)~\cite{ortizjimenez2023}, and OrthoReg~\cite{liu2026understanding}.

\noindent\textbf{Evaluation metrics.}
Consistent with~\cite{ilharco2022editing}, we report Absolute Accuracy (Abs.Acc.), and Normalized Accuracy (Norm.Acc.). A single coefficient $\alpha$ (\textit{i.e.}, $\alpha_t=\alpha$ for all $t$ in Eq.~\ref{eq:wd})) is applied to $\sum_t \tau_t$ and selected on the validation sets.

A full description of the datasets, fine-tuning protocol, baselines, and evaluation details is provided in Appendix.

\begin{figure}[t]
\centering

\includegraphics[width=\linewidth]{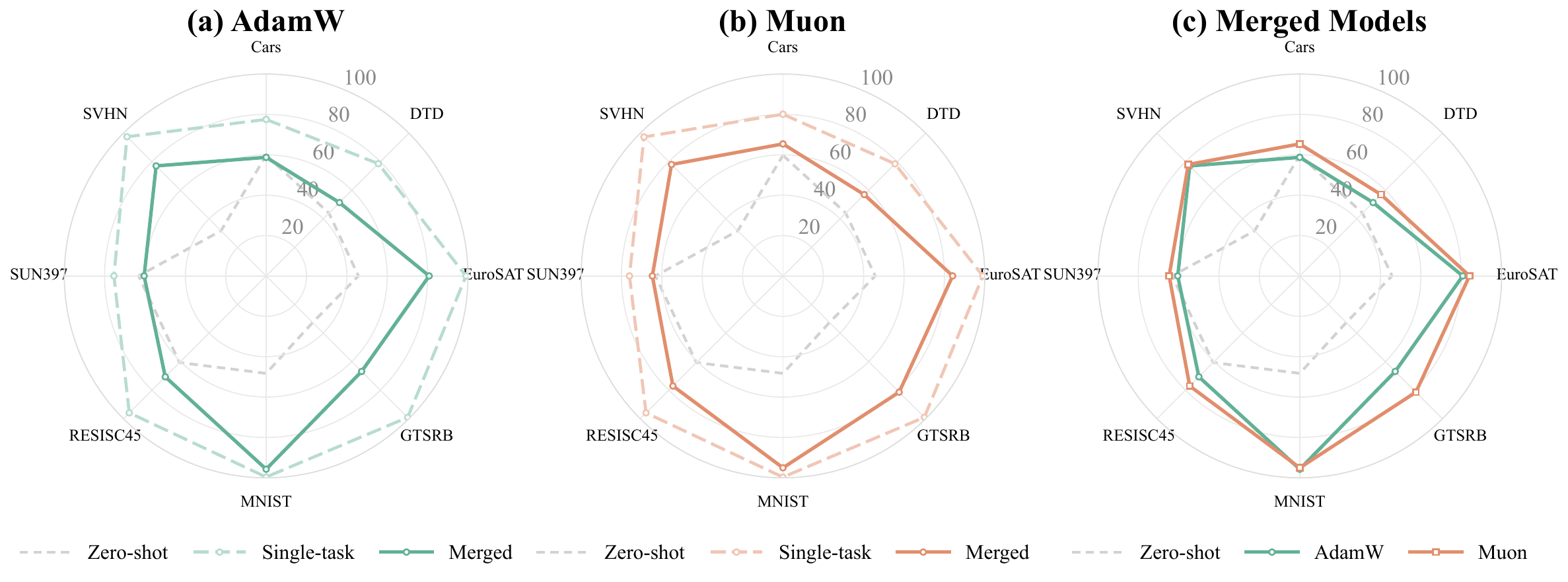}
\caption{Per-task accuracy on ViT-B/32 under AdamW and Muon. (a, b) Grey dashed: zero-shot backbone. Light dashed: accuracy of each independently fine-tuned expert on its own task. Bold solid: accuracy of the merged model. (c) Direct comparison of the two merged models. Single-task accuracies are nearly identical between optimizers, but the merged accuracies expand substantially under Muon.}
\label{fig:radar_per}
\vspace{-10pt}
\end{figure}

\subsection{Performance and Analysis}
\noindent\textbf{Overall performance.}
Table~\ref{tab:mm_main} reports the task addition results. For non-linear fine-tuning, replacing AdamW with Muon improves Abs.~Acc by $+5.02$, $+3.42$, and $+3.82$ points on ViT-B/32, ViT-B/16, and ViT-L/14, respectively. Norm.Acc. shows similar gains.
The improvement carries over when Muon is combined with existing pre-merging methods: OrthoReg + Muon gives the best result on every backbone ($78.05\%$, $82.21\%$, and $89.83\%$ Abs.Acc.), and TTA + Muon improves over TTA + AdamW by $+2.73$ points on ViT-B/32. The gain is therefore not tied to a specific fine-tuning parameterization or a specific regularizer. It comes from the optimizer itself, and it holds across model scales and across pre-merging strategies, in line with Theorem~\ref{thm:bound}.

\begin{figure}[t]
\centering
\includegraphics[width=\linewidth]{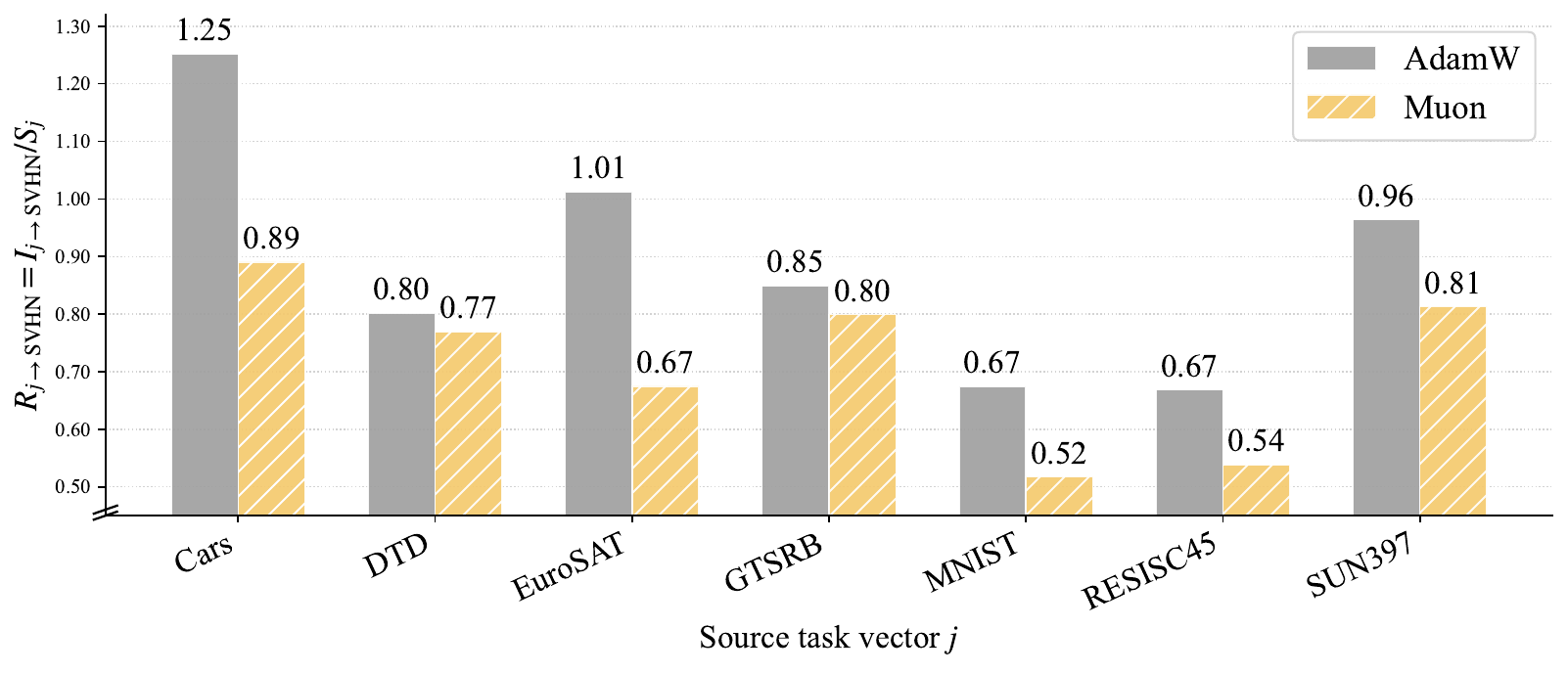}

\caption{Relative interference $R_{j \to \text{SVHN}} = I_{j \to \text{SVHN}} / S_j$ on ViT-B/32 for each source task vector $j$. Lower is better. Muon consistently reduces relative interference.}
\label{fig:relative_interference}
\vspace{-15pt}
\end{figure}

\noindent\textbf{Per-Task performance analysis.}
Muon does not produce stronger single-task experts; it produces more mergeable task vectors. 
As shown in Figure~\ref{fig:radar_per}, on ViT-B/32, the average accuracy of the eight independently fine-tuned experts on their own tasks differs by only $0.23$ points between AdamW and Muon, while after merging, the absolute-accuracy gap widens to $+5.02$ points. Almost the entire improvement comes from the merging step, not from fine-tuning. 

This is consistent with Theorem~\ref{thm:bound}: Muon controls $\|\Delta W_\ell\|_2$ and reduces cross-task interference at merging.

\noindent\textbf{Weight disentanglement.} 
We now verify whether Muon produces more disentangled task vectors in the sense of~\citet{ortizjimenez2023}. Following their protocol, we sweep $(\alpha_i, \alpha_j)$ and, at each point, measure the mean prediction disagreement.

Lower disagreement indicates better disentanglement.
Figure~\ref{fig:wd} shows the result on Cars + RESISC45 (ViT-B/32). Under AdamW the low-disagreement region is confined to a narrow band near the origin. Under Muon it expands substantially along both axes. This is exactly the disentanglement error $\delta_{\text{DE}}$ that Eq.~\eqref{eq:wd} formalizes and Theorem~\ref{thm:bound} upper-bounds: shrinking $\|\Delta W_\ell\|_2$ shrinks $\delta_{\text{DE}}$, so the merging is better.

\begin{figure}[t]
\centering
\includegraphics[width=0.9\linewidth]{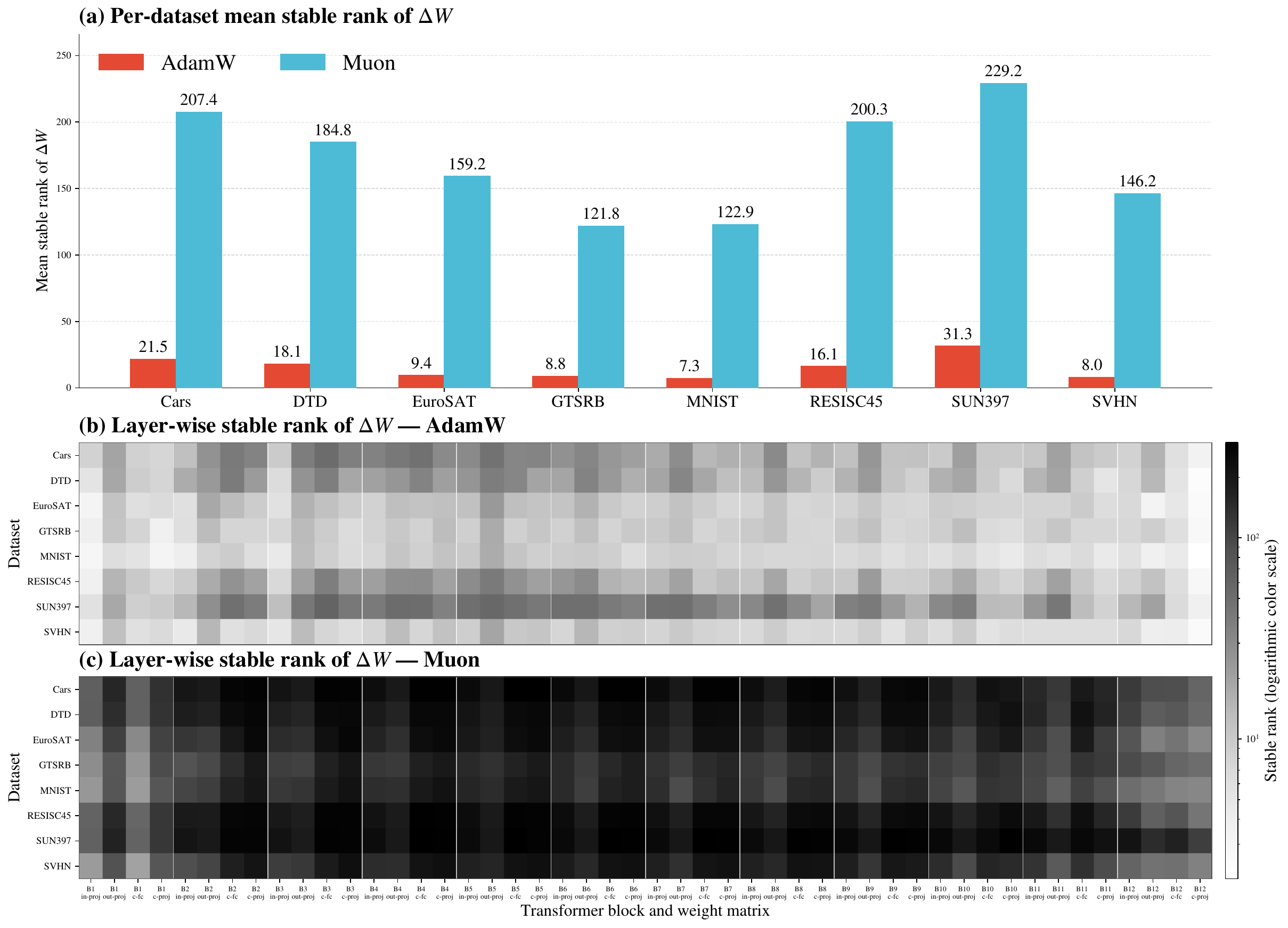}
\caption{Stable rank of $\Delta W$ on CLIP ViT-B/32 under AdamW and Muon. (a) Per-dataset mean stable rank. Muon increases the mean by roughly $11\times$ on every task. (b, c) Layer-wise stable rank for every matrix type of all blocks, under AdamW and Muon respectively.}
\label{fig:stable_rank}
\vspace{-15pt}
\end{figure}

\noindent\textbf{Task interference.} 
We now probe the layer-summed interference term $\sum_\ell \langle \Delta W_\ell, J_\ell(x)\rangle_F$ that Eq.~\eqref{eq:unified} formalizes and Theorem~\ref{thm:bound} upper-bounds. In the MM setting $\Delta W = \alpha_j \tau_j$ (Eq. (4)), and under the NTK linearization this sum equals the model-level product $\tau_j^\top J(x)$ up to the scaling $\alpha_j$. 
For each source task $j$ we therefore measure the self-effect on its own domain, $S_j = \mathbb{E}_{x \sim \mathcal{D}_j}[\|\tau_j^\top J(x)\|_2]$, and the cross-task interference on a target $t \neq j$, $I_{j \to t} = \mathbb{E}_{x \sim \mathcal{D}_t}[\|\tau_j^\top J(x)\|_2]$. The relative interference $R_{j \to t} = I_{j \to t}/S_j$ quantifies how much of $\tau_j$'s functional effect leaks onto task $t$.
Figure~\ref{fig:relative_interference} reports $R_{j \to \text{SVHN}}$ on ViT-B/32. Muon reduces the relative interference on every source task, with the largest drops on EuroSAT (1.01 $\to$ 0.67), directly validating Theorem~\ref{thm:bound}.

\subsection{Why Muon Merges Better: A Spectral View}
\noindent\textbf{Mode-wise contribution to task interference.}
Figure~\ref{fig:contri} reports the per-mode contribution $\text{Contrib}_i(x)$. Under AdamW the top-10\% of singular modes account for $63.5\%$ of the total interference. Under Muon the same fraction is only $29.0\%$ and the curve is nearly flat. 

This confirms Remark~\ref{rem:tightness}: for AdamW the per-mode contribution concentrates precisely on the leading singular modes. These are exactly the modes suppressed by Muon's spectrum flattening at the matched Frobenius scale. Combined with the direct relative-interference reduction reported in Figure~\ref{fig:relative_interference}, this shows that shrinking the bound of Theorem~\ref{thm:bound} reduces the empirical interference itself.

\noindent\textbf{Stable rank of the update.} For a matrix $A$ with singular values $\sigma_1 \ge \sigma_2 \ge \cdots \geq 0$, we investigate the stable rank~\cite{Rudelson2007} $\mathrm{srank}(A) \;:=\; \|A\|_F^2 \,/\, \|A\|_2^2$. By Theorem~\ref{thm:spectral_steepest}, the spectral-norm steepest descent solution has uniform singular values and thus substantially larger stable rank than the updates from spectrum-unconstrained optimizers. Muon, which approximates this solution, should therefore inherit this property.
As shown in Figure~\ref{fig:stable_rank}, the per-dataset mean of stable ranks rises from $15.1$ under AdamW to $171.5$ under Muon (about $11\times$, panel (a)), and the layer-wise maps (panels (b), (c)) show this increase is uniform across all blocks.
Since $\|\Delta W_\ell\|_2 = \|\Delta W_\ell\|_F / \sqrt{\mathrm{srank}(\Delta W_\ell)}$, higher stable rank at a given Frobenius scale forces a smaller spectral norm. The gap in Figure~\ref{fig:stable_rank} thus underlies the reductions in relative interference (Figure~\ref{fig:relative_interference}), disentanglement error (Figure~\ref{fig:wd}), and merged-model accuracy (Table~\ref{tab:mm_main}).
\section{Experiments on Continual Learning}\label{sec:exp_cl}
\subsection{Experimental Setup}

\begin{table}[t]
\centering
\small
\setlength{\tabcolsep}{4pt}
\begin{tabular}{llccccccc}
\toprule
& & & \multicolumn{2}{c}{AdamW} & \multicolumn{2}{c}{Muon} & \multicolumn{2}{c}{Gain} \\
\cmidrule(lr){4-5} \cmidrule(lr){6-7} \cmidrule(lr){8-9}
 & \textit{B} & Tasks & Avg.\,$\uparrow$ & Last\,$\uparrow$ & Avg.\,$\uparrow$ & Last\,$\uparrow$ & Avg. & Last \\
\midrule
\multirow{3}{*}{CIL}
 & \textit{I} & 10 & 79.43 & 71.22 & 81.55 & 75.12 & \textbf{+2.12} & \textbf{+3.90} \\
 & \textit{I} & 20 & 71.04 & 64.92 & 75.19 & 68.43 & \textbf{+4.15} & \textbf{+3.51} \\
 & \textit{C}  & 50 & 62.04 & 47.87 & 64.39 & 51.01 & \textbf{+2.35} & \textbf{+3.14} \\
\midrule
\multirow{3}{*}{TIL}
 & \textit{C}  &  5 & 93.92 & 91.81 & 95.61 & 93.45 &\textbf{ +1.69}          & \textbf{+1.64} \\
 & \textit{C}  & 10 & 91.68 & 90.15 & 95.07 & 91.94 & \textbf{+3.39} & \textbf{+1.79} \\
 & \textit{C}  & 20 & 91.81 & 90.44 & 96.07 & 93.64 & \textbf{+4.26} & \textbf{+3.20} \\
\bottomrule
\end{tabular}

\caption{LoRA baseline with AdamW vs.\ Muon on CIL and TIL protocols. Muon uniformly improves both metrics. \textit{B} (benchmark): \textit{I} = ImageNet-R, \textit{C} = CIFAR-100.}\label{tab:lora_cl}

\end{table}

\textbf{Choice of base method.} 
We evaluate Muon on two baselines that differ sharply in how much interference their design already suppresses. Plain LoRA fine-tuning inserts low-rank adapters into the frozen CLIP image encoder without any explicit CL mechanism, so any change in retention is attributed to the update geometry alone. MoE-Adapters4CL~\cite{moeadapter2024} is a strong parameter-efficient framework that already suppresses interference. Both baselines are trained on matrix-valued adapter parameters, so "+ Muon" denotes exactly the same architecture, data order, and training budget as the AdamW baseline, with only the optimizer changed. Full protocol details are in Appendix.

\noindent\textbf{Benchmarks and protocols.}
For task-incremental learning (TIL), we adopt the standard split of CIFAR-100 and ImageNet-R into 5/10/20 disjoint tasks. For class-incremental learning (CIL), we use CIFAR-100 (10/20/50 tasks), TinyImageNet (100-class base task plus 5/10/20 incremental steps), and ImageNet-R (10 tasks $\times$ 20 classes), reporting cumulative top-1 accuracy averaged over all stages (\emph{Avg}) and after the final task (\emph{Last}). For multi-domain task-incremental learning (MTIL), we follow the 11 sequential datasets proposed by~\citet{ZSCL} and report \emph{Transfer}, \emph{Average}, and \emph{Last} metrics. Full benchmark definitions and metric formulas are given in Appendix.

\begin{table}[t]
\small
\centering
\setlength{\tabcolsep}{3.5pt}
\renewcommand{\arraystretch}{0.95}
\begin{tabular}{@{}llcccc@{}}
\toprule
\multirow{2}{*}{Benchmark}
& \multirow{2}{*}{Protocol}
& \multicolumn{2}{c}{AdamW}
& \multicolumn{2}{c}{Muon} \\
\cmidrule(lr){3-4}
\cmidrule(lr){5-6}
& & Avg. $\uparrow$ & Last $\uparrow$
  & Avg. $\uparrow$ & Last $\uparrow$ \\
\midrule
\multirow{3}{*}{CIFAR-100}
& 10 tasks & 85.30 & 77.56 & \textbf{85.98} & \textbf{78.82} \\
& 20 tasks & 84.36 & 76.61 & \textbf{84.87} & \textbf{77.27} \\
& 50 tasks & 83.39 & 73.76 & \textbf{83.84} & \textbf{74.02} \\
\midrule
\multirow{3}{*}{TinyImageNet}
& 5 steps  & 80.85 & 77.04 & \textbf{81.62} & \textbf{77.57} \\
& 10 steps & 80.26 & 75.79 & \textbf{81.14} & \textbf{76.75} \\
& 20 steps & 79.98 & 75.68 & \textbf{80.67} & \textbf{75.86} \\
\midrule
ImageNet-R
& $10{\times}20$
& 87.13 & 82.40
& \textbf{87.92} & \textbf{83.27} \\
\bottomrule
\end{tabular}
\vspace{-5pt}
\caption{CIL results on MoE-Adapters4CL. Replacing AdamW with Muon improves both Avg. and Last across all seven protocols.}
\vspace{-10pt}
\label{tab:cl_moe_cil}
\end{table}

\begin{table}[t]
\small
\centering
\setlength{\tabcolsep}{7pt}
\begin{tabular}{lccc}
\toprule
Method & Transfer $\uparrow$ & Average $\uparrow$ & Last $\uparrow$ \\
\midrule
AdamW & 66.4 & 76.9 & 86.3 \\
Muon  & \textbf{67.1} & \textbf{78.3} & \textbf{87.9} \\
\bottomrule
\end{tabular}
\vspace{-5pt}
\caption{11-task MTIL evaluation on MoE-Adapters4CL. Replacing AdamW with Muon improves Transfer, Average,
and Last by $+0.7$, $+1.4$, and $+1.6$ points, respectively. Complete per-task results are provided in
Appendix Table.}
\label{tab:cl_mtil}
\vspace{-5pt}
\end{table}

\subsection{Performance and Analysis}

\textbf{LoRA baseline.} As shown in Table~\ref{tab:lora_cl}, on plain LoRA, which contains no interference-mitigation mechanism, replacing AdamW with Muon improves Avg. by +4.15 on ImageNet-R (20-task CIL) and +4.26 on CIFAR-100 (20-task TIL). The improvement is uniformly positive across every task horizon we test and grows with sequence length on CIFAR-100 ($+1.69 \rightarrow +3.39 \rightarrow +4.26$ for 5/10/20 tasks), directly confirming the mechanism of Theorem~\ref{thm:bound} in the CL setting.

\noindent\textbf{MoE-Adapters4CL baseline.} MoE-Adapters4CL is a mature CL framework that already suppresses interference. Even so, replacing AdamW with Muon improves both Avg and Last in every one of the seven CIL protocols on CIFAR-100, TinyImageNet, and ImageNet-R (Table~\ref{tab:cl_moe_cil}). The benefit is uniformly positive and persists from short to long sequences. On the 11-task MTIL benchmark (Table~\ref{tab:cl_mtil}), Muon improves all three metrics over AdamW. Complete results and comparisons against additional baselines are provided in Appendix.

\begin{figure}[t]
\centering
\includegraphics[width=\linewidth]{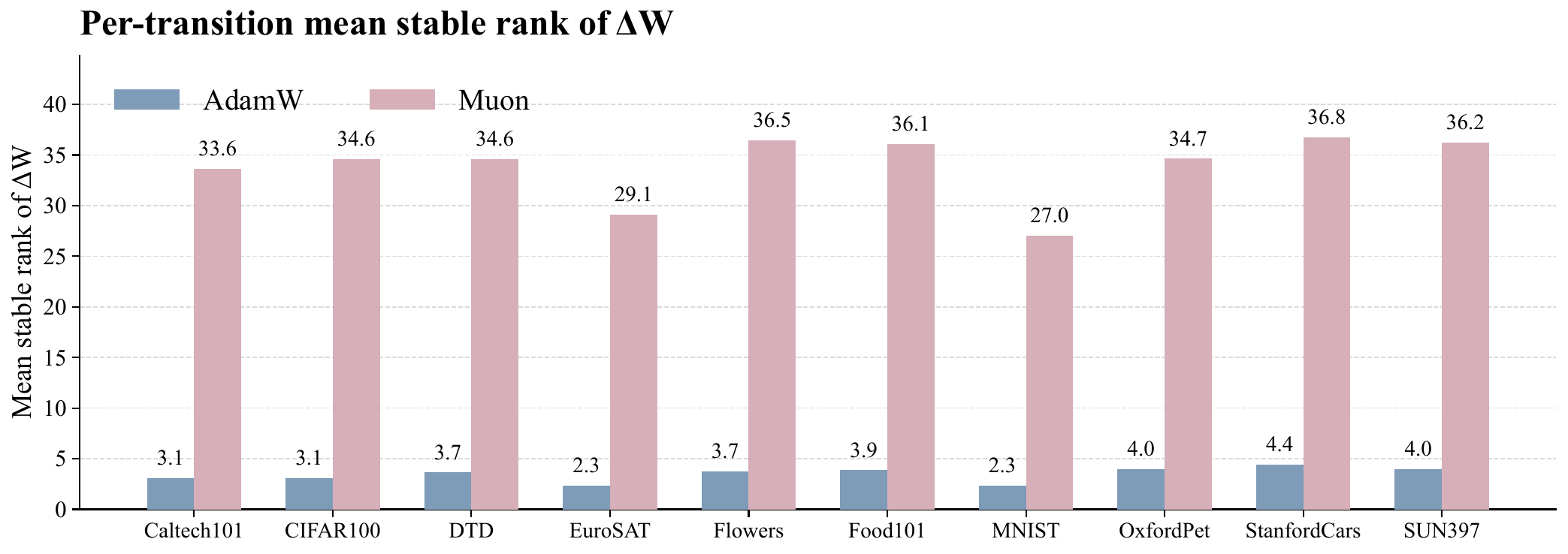}
\caption{Mean stable rank of $\Delta W$ on the MTIL sequence. AdamW: $[2.3, 4.4]$ across all datasets; Muon: $[27.0, 36.8]$, roughly $10\times$ higher and uniform across datasets. Higher stable rank at fixed Frobenius scale forces a smaller spectral norm and a tighter Theorem~\ref{thm:bound} bound at every transition.}
\label{fig:stable_rank-cl}

\end{figure}

\noindent\textbf{Mode-wise contribution.} Figure~\ref{fig:contri-cl} aggregates $\mathrm{Contrib}_i(x)$ over all transitions of the MTIL sequence. Under AdamW the per-mode contribution decays sharply and concentrates on the leading modes. Under Muon it stays flat. This is the empirical premise of Remark~\ref{rem:tightness} instantiated in the CL setting.

\noindent \textbf{Stable rank.} Figure~\ref{fig:stable_rank-cl} reports $\mathrm{srank}(\Delta W)$ averaged over transitions per source dataset. Muon raises the stable rank by roughly $10\times$. Since $\|\Delta W_\ell\|_2 = \|\Delta W_\ell\|_F / \sqrt{\mathrm{srank}(\Delta W_\ell)}$, the same Frobenius energy is spread over $10\times$ more effective modes, directly tightening the Theorem~\ref{thm:bound} bound.

\section{Conclusion}

We unified catastrophic forgetting in continual learning and weight-disentanglement error in model merging as two instances of the same task interference $\langle \Delta W_\ell, J_\ell(x)\rangle_F$, and bounded it by the product of a nuclear norm fixed by the backbone and data, and an optimizer-controllable spectral norm $\|\Delta W_\ell\|_2$. We show that Muon, as spectral-norm steepest descent, keeps this factor uniformly small by construction, tightening the bound for both paradigms through a single change of optimizer. Empirically, replacing AdamW with Muon yields uniformly positive gains on both continual learning and model merging. 
For future work, we plan to extend the interference bound beyond the NTK regime, and to validate the spectral-norm mechanism on billion-parameter foundation models and longer task sequences.

\bibliography{references}

\appendix
\section{Detailed Proofs}
\label{proof:thm1}

\subsection{Proof of Theorem~1}

To simplify notation, let $A = \Delta W_\ell$ and $B = J_\ell(x)$. Let the
reduced SVD~\cite{matrixgolub2013} of $B$ be
\[
    B = \sum_{i=1}^{k} \sigma_i\, u_i v_i^\top,
\]
where $\sigma_i > 0$ and $\{u_i\}$, $\{v_i\}$ are the left and right singular
vectors. By the definition of the nuclear norm,
$\|B\|_* = \sum_{i=1}^{k} \sigma_i$.

By linearity of the Frobenius inner product,
\begin{equation}
    \langle A,\, B \rangle_F
    = \sum_{i=1}^{k} \sigma_i \left\langle A,\, u_i v_i^\top \right\rangle_F.
\end{equation}
Using $\langle A, B \rangle_F = \mathrm{tr}(A^\top B)$ and the cyclic
invariance of the trace, each term simplifies as
\begin{equation}
    \left\langle A,\, u_i v_i^\top \right\rangle_F
    = \mathrm{tr}(A^\top u_i v_i^\top)
    = v_i^\top A^\top u_i
    = u_i^\top A v_i,
\end{equation}
so
\begin{equation}
    \langle A,\, B \rangle_F
    = \sum_{i=1}^{k} \sigma_i \cdot u_i^\top A v_i.
    \label{eq:svd_expansion}
\end{equation}
By Cauchy--Schwarz and the operator definition of the spectral norm
$\|A\|_2 = \sup_{\|v\|_2 = 1} \|Av\|_2$, together with
$\|u_i\|_2 = \|v_i\|_2 = 1$, each scalar term satisfies
\begin{equation}
    \left|u_i^\top A v_i\right|
    \leq \|u_i\|_2 \cdot \|A v_i\|_2
    \leq \|A\|_2.
    \label{eq:matrix_element}
\end{equation}
Taking the absolute value of Eq.~\eqref{eq:svd_expansion} and applying
Eq.~\eqref{eq:matrix_element} yields
\begin{equation}
    \left|\langle A,\, B \rangle_F\right|
    \leq \sum_{i=1}^{k} \sigma_i \cdot \left|u_i^\top A v_i\right|
    \leq \|A\|_2 \sum_{i=1}^{k} \sigma_i
    = \|A\|_2 \cdot \|B\|_*.
\end{equation}
Substituting back $A = \Delta W_\ell$ and $B = J_\ell(x)$ completes the
proof.\hfill$\square$

\subsection{An Optimizer-Independent Bound on $\|J_\ell(x;\theta_{\text{ref}})\|_*$}

We now make explicit the quantity referenced in Eq.~(6) of the main text. The following two assumptions are standard in the fine-tuning literature~\cite{NTK,liu2020linearity}.

\begin{assumption}[Local Lipschitz continuity of the layer-$\ell$ Jacobian]\label{assumption:local_lip}
    There exist a constant $L_{J,\ell}>0$ and an open neighborhood $\mathcal U\ni\theta_0$ such that for every $x\in\mathrm{supp}(\mathcal D_s)$ and every $\theta\in\mathcal U$,
\begin{equation}
    \big\|J_\ell(x;\theta)-J_\ell(x;\theta_0)\big\|_{*}\;\le\;L_{J,\ell}\,\|\theta-\theta_0\|_2.
\end{equation}
\end{assumption}

Assumption~\ref{assumption:local_lip} is implied by a bounded Hessian of $f$ at $\theta_0$, which holds under the smooth-activation, bounded-input regime standard in NTK analyses~\cite{NTK}.

\begin{assumption}[Bounded fine-tuning radius]\label{assumption:bound_ft}
    There exists $R>0$ such that
\begin{equation}
    \|\theta_{\text{ref}}-\theta_0\|_2\;\le\;R,
\end{equation}
where $\theta_{\text{ref}}=\theta^\star_{t-1}$ in CL and $\theta_{\text{ref}}=\theta_0+\alpha_s\tau_s$ in MM.
\end{assumption}

Assumption~\ref{assumption:bound_ft} is enforced by any fine-tuning protocol with a finite step budget, learning rate schedule, and (implicit or explicit) weight decay. It is the same regime under which task arithmetic and NTK-linearization arguments are typically invoked~\cite{ortizjimenez2023}.

\begin{proposition}[Optimizer-independent bound on the Jacobian nuclear norm]\label{prop:jac_quanity}
    Under Assumptions 1 and 2,
    \begin{equation}
        \mathbb E_{x\sim\mathcal D_s}\big[\|J_\ell(x;\theta_{\text{ref}})\|_{*}\big]\;\le\;\mathcal J_\ell(\mathcal D_s;\theta_0)\;+\;L_{J,\ell}\,R,
    \end{equation}
where $\mathcal J_\ell(\mathcal D_s;\theta_0):=\mathbb E_{x\sim\mathcal D_s}\!\big[\|J_\ell(x;\theta_0)\|_{*}\big]$. 

\end{proposition}

\begin{proof}
    By the triangle inequality applied to the nuclear norm,
\begin{equation}
     \|J_\ell(x;\theta_{\text{ref}})\|_{*}
\;\le\;
\|J_\ell(x;\theta_0)\|_{*}
\;+\;
\|J_\ell(x;\theta_{\text{ref}})-J_\ell(x;\theta_0)\|_{*}.
\end{equation}

By Assumption~\ref{assumption:local_lip} and Assumption~\ref{assumption:bound_ft}, the second term is bounded by $L_{J,\ell}\|\theta_{\text{ref}}-\theta_0\|_2\le L_{J,\ell}R$ for every $x\in\mathrm{supp}(\mathcal D_s)$. Taking expectation over $x\sim\mathcal D_s$ yields the claim.
\end{proof}

 The three quantities on the right-hand side of Proposition~\ref{prop:jac_quanity} admit clean, non-optimizer interpretations. First, $\mathcal J_\ell(\mathcal D_s;\theta_0)$ is a property of the pre-trained representation: it measures how much of the layer-$\ell$ Jacobian energy $\theta_0$ places on the source-task distribution. Second, $L_{J,\ell}$ is a curvature constant of the model at $\theta_0$, determined by architecture and activation choice. Third, $R$ is set by the fine-tuning protocol (step budget, learning rate, weight decay) and is matched across optimizers in our experiments. Different optimizers may reach different $\theta_{\text{ref}}$ within the ball of radius $R$, but none of the three quantities above changes. 

\subsection{Spectral--Nuclear Duality and the Steepest Descent Solution}
\label{proof:spectral_dual}

This section derives the closed-form solution of
Eq.(9) in the main text. The derivation follows
the standard argument for spectral-norm steepest
descent~\cite{bernstein2024}.

\textbf{Spectral--nuclear duality.} We first record the fact that the
nuclear norm is dual to the spectral norm: for any matrix $G$,
\begin{equation}
    \|G\|_* = \max_{\|T\|_2 \leq 1} \langle G,\, T \rangle_F,
    \label{eq:duality}
\end{equation}
and, if $G = U \Sigma V^\top$ is a reduced SVD, the maximum is attained
uniquely (up to the null space of $G$) at $T = U V^\top$.

\emph{Proof of Eq.~\eqref{eq:duality}.} Substituting the SVD of $G$ and
using the same manipulation as in
Eqs.~\eqref{eq:svd_expansion}--\eqref{eq:matrix_element},
\begin{equation}
    \langle G,\, T \rangle_F
    = \sum_i \sigma_i \cdot u_i^\top T v_i
    \;\leq\; \sum_i \sigma_i \cdot \|T\|_2
    \;\leq\; \sum_i \sigma_i
    = \|G\|_*,
\end{equation}
where the last inequality uses $\|T\|_2 \leq 1$. The upper bound is
attained by $T = UV^\top$: since $U^\top U = V^\top V = I_r$, we have
$u_i^\top (UV^\top) v_i = 1$ for every $i$, so
$\langle G, UV^\top \rangle_F = \sum_i \sigma_i = \|G\|_*$. When $G$ has
full rank, uniqueness follows from the requirement
$u_i^\top T v_i = 1$ for all $i$ combined with
$\|T\|_2 \leq 1$~\cite{bernstein2024}.

\textbf{Optimal solution.} Write the candidate update as
$\Delta W_\ell = -c \cdot T$ with $c \geq 0$ and $\|T\|_2 = 1$. The
objective in Eq.(9) becomes
\begin{equation}
    f(c, T) = -c\,\langle G_\ell,\, T \rangle_F + \tfrac{\lambda}{2} c^2.
\end{equation}
For any fixed $c \geq 0$, this is minimized over $T$ by maximizing
$\langle G_\ell, T \rangle_F$ subject to $\|T\|_2 = 1$; by
Eq.~\eqref{eq:duality} the maximum equals $\|G_\ell\|_*$ and is attained
at $T^\star = U_\ell V_\ell^\top$. Substituting gives
\begin{equation}
    f(c, T^\star) = -c\,\|G_\ell\|_* + \tfrac{\lambda}{2} c^2,
\end{equation}
which is minimized over $c \geq 0$ at $c^\star = \|G_\ell\|_* / \lambda$.
Therefore
\begin{equation}
    \Delta W_\ell^\star
    = -\frac{\|G_\ell\|_*}{\lambda}\,\cdot\, U_\ell V_\ell^\top,
\end{equation}
as claimed.

\subsection{The Muon Algorithm}
\label{app:muon-alg}

For completeness, we record the Muon update at step $t$ for a weight
matrix $W_\ell$; see~\citet{jordan2024muon} for the full description.

Given the current-step gradient $G_\ell^{(t)}$ and momentum coefficient
$\beta \in [0, 1)$, Nesterov momentum smoothing produces the effective
gradient
\begin{equation}\small
    M_\ell^{(t)} = \beta M_\ell^{(t-1)} + G_\ell^{(t)},
    \qquad
    \tilde{M}_\ell^{(t)} = \beta M_\ell^{(t)} + G_\ell^{(t)}.
\end{equation}
Setting $X_0 = \tilde{M}_\ell^{(t)} / \|\tilde{M}_\ell^{(t)}\|_F$, the
Newton--Schulz iteration
\begin{equation}\small
    X_{k+1} = \tfrac{3}{2} X_k - \tfrac{1}{2} X_k X_k^\top X_k,
    \qquad k = 0, 1, \ldots, K-1,
    \label{eq:ns_iter}
\end{equation}
drives $X_K$ to the polar factor $U_\ell V_\ell^\top$ of
$\tilde{M}_\ell^{(t)}$; convergence is analyzed in
Appendix~\ref{proof:ns_convergence}. The weight update is
$W_\ell^{(t+1)} = W_\ell^{(t)} - \eta \cdot X_K$. Following~\citet{jordan2024muon},
we use $\beta = 0.95$ and $K = 5$ throughout.

\subsection{Convergence of the Newton--Schulz Iteration}
\label{proof:ns_convergence}

The Newton--Schulz iteration in Eq.~\eqref{eq:ns_iter} is a classical
scheme for the polar factor~\cite{bernstein2024}. Its convergence to
$U_\ell V_\ell^\top$ follows from a scalar reduction.

\textbf{Equivariance of the singular vectors.} Let
$X_k = \hat{U}\hat{\Sigma}_k\hat{V}^\top$ be the SVD of the $k$-th
iterate. Substituting into Eq.~\eqref{eq:ns_iter},
\begin{equation}
    X_{k+1}
    = \tfrac{3}{2}\hat{U}\hat{\Sigma}_k\hat{V}^\top
      - \tfrac{1}{2}\hat{U}\hat{\Sigma}_k^3\hat{V}^\top
    = \hat{U}\!\left(\tfrac{3}{2}\hat{\Sigma}_k - \tfrac{1}{2}\hat{\Sigma}_k^3\right)\!\hat{V}^\top.
\end{equation}
The singular vector matrices $\hat{U}$ and $\hat{V}$ are therefore
invariant along the iteration, and each singular value $\hat{\sigma}_i$
evolves independently under the scalar map
\begin{equation}
    f(x) = \tfrac{3}{2}x - \tfrac{1}{2}x^3.
\end{equation}

\textbf{Fixed-point analysis.} The unique positive fixed point of $f$ is
$x = 1$. Since $f'(x) = \tfrac{3}{2}(1 - x^2)$, we have $f'(1) = 0$, so
$x = 1$ is an attracting fixed point. For any $x \in (0, \sqrt{3})$, the
sequence $\{f^{(k)}(x)\}_{k \geq 0}$ converges monotonically to $1$.

\textbf{Initialization.} The Frobenius normalization
$X_0 = \tilde{M}_\ell^{(t)}/\|\tilde{M}_\ell^{(t)}\|_F$ ensures
\begin{equation}
    \hat{\sigma}_i^{(0)}
    \leq \|X_0\|_2
    \leq \|X_0\|_F = 1 < \sqrt{3},
\end{equation}
so every initial singular value lies in the basin of attraction
$(0, \sqrt{3})$~\cite{bernstein2024}, guaranteeing convergence.

\textbf{Rate.} Since $f'(1) = 0$ and $f''(1) = -3 \neq 0$, convergence is
of second order. In practice $K = 5$ to $10$ steps drive the singular
values to within $10^{-3}$ of unity, so
$X_K \approx \hat{U}\hat{V}^\top = U_\ell V_\ell^\top$.

\section{Detailed Related Work}
\subsection{Continual Learning}
Continual learning aims to incrementally acquire knowledge from a sequence of tasks while avoiding catastrophic forgetting~\cite{catastrophicmccloskey1989,mcclelland1995}. Existing methods can be broadly grouped into five categories~\cite{wang2024survey,li2025libcontinual}.

Regularization-based methods mitigate forgetting by adding explicit penalty terms to the loss function that constrain updates to parameters deemed important for previously learned tasks~\cite{EWC,zenke2017,Li2018LwF,liu2026continual}.Replay-based methods preserve past knowledge by storing or generating representative samples from old tasks and replaying them during the learning of new tasks~\cite{iCaRL,lopezpaz2017,shin2017}.
Optimization-based methods explicitly manipulate the gradient update process to prevent interference with previously learned knowledge, often by projecting gradients into subspaces that are orthogonal to those of old tasks~\cite{chaudhry2018,farajtabar2020,saha2021,liu2025last,Kang2025}.
Representation-based methods leverage robust and stable feature representations to improve resistance to catastrophic forgetting, including approaches based on self-supervised learning and pre-trained model prompting~\cite{madaan2021,pham2021,wang2022l2p}.
Architecture-based methods explicitly allocate separate parameter subspaces or sub-networks to different tasks, thereby avoiding inter-task interference by design~\cite{rusu2016,mallya2018,serra2018}.

\subsection{Model Merging.}
Model merging combines multiple independently fine-tuned models into a 
single unified model directly in weight space, without requiring 
additional training or data access~\cite{wortsman2022model,matena2022fisher}. 

Existing solutions can be broadly classified into during-merging and 
pre-merging methods~\cite{ilharco2023,yang2024adamerging}. During-merging 
methods design sophisticated combination algorithms applied to 
already-trained models, such as resolving sign conflicts~\cite{yadav2023ties}, 
learning adaptive merging coefficients~\cite{yang2024adamerging}, and 
weighting parameter contributions via Fisher information~\cite{matena2022fisher}. 
Pre-merging methods instead modify the fine-tuning process to produce 
more mergeable models, for example by fine-tuning in the linearized 
tangent space~\cite{ortizjimenez2023} or applying sparsity-inducing 
regularization~\cite{yu2024language}. Our work belongs to the pre-merging 
category. Specifically, we analyze model merging from the perspective 
of the Muon optimizer, which orthogonalizes 
gradient updates via Newton-Schulz iteration and implicitly enforces 
spectral norm constraints on weight updates. 

\subsection{Muon Optimizer}

The Muon optimizer~\cite{jordan2024muon} is a recently proposed training algorithm for the hidden-layer weight matrices of neural networks. At each update step, it computes Nesterov momentum on the raw gradient and subsequently 
orthogonalizes the resulting matrix via Newton-Schulz iterations, yielding an update that approximates steepest descent under the spectral norm. By constraining parameter changes to lie on the orthogonal group, Muon encourages 
more balanced gradient flow across the singular value spectrum of each weight matrix, and has demonstrated competitive convergence relative to AdamW in large-scale language model pretraining benchmarks at comparable computational cost.Recent work has further scaled Muon to billion-parameter language models~\cite{lu2026muonogd} and explored its theoretical connections to spectral-norm steepest descent~\cite{bernstein2024}, establishing 
orthogonalization-based updates as a practically viable and theoretically grounded paradigm for modern large-scale training.

However, all prior work studies Muon exclusively through the lens of optimization efficiency and convergence speed in single-task settings, leaving its implications for sequential learning largely unexplored. In this work, we take a different perspective and show that the orthogonal update structure imposed by Muon provides a strict theoretical bound on 
catastrophic forgetting in continual learning. Specifically, by analyzing weight perturbations under the spectral norm, we demonstrate that Muon's updates are provably bounded in spectral norm magnitude, which directly constrains the deviation of the model's behavior on previously learned tasks. This spectral-norm perspective not only offers new theoretical insight into why geometry-aware optimizers may be inherently beneficial for sequential learning, but also motivates a principled optimizer-centric approach to forgetting mitigation that is complementary to existing regularization- and replay-based methods.
\section{More Details of Model Merging Experiments}
\label{app:mm_setup}

This appendix expands the eight-task addition protocol summarized in Section ``Model Merging Experiments".

\subsection{Datasets and Tasks}

We follow the standard eight-task addition benchmark of Model Merging~\cite{ortizjimenez2023,liu2026understanding}, which independently fine-tunes one expert on each of the following image-classification datasets.
\begin{itemize}
    \item \textbf{SUN397}~\citep{xiao2010sun}: large-scale scene recognition
    across $397$ scene categories.
    \item \textbf{Stanford Cars}~\citep{krause2013cars}: fine-grained
    recognition over $196$ car models.
    \item \textbf{RESISC45}~\citep{cheng2017resisc45}: aerial scene
    classification with $45$ remote-sensing classes.
    \item \textbf{EuroSAT}~\citep{helber2019eurosat}: land-use and land-cover
    classification from Sentinel-$2$ satellite imagery.
    \item \textbf{SVHN}~\citep{netzer2011svhn}: street-view digit
    recognition in natural images.
    \item \textbf{GTSRB}~\citep{stallkamp2011gtsrb}: German traffic-sign
    recognition across $43$ classes.
    \item \textbf{MNIST}~\citep{lecun1998mnist}: handwritten digit
    recognition.
    \item \textbf{DTD}~\citep{cimpoi2014dtd}: texture classification
    with $47$ describable texture categories.
\end{itemize}
The eight tasks span scene recognition, fine-grained classification, remote
sensing, satellite land-use classification, digit recognition, traffic-sign
recognition, and texture classification, and therefore stress cross-task
interference across largely disjoint visual domains. Standard train, validation,
and test splits are used throughout.

\subsection{Backbones and Fine-tuning Protocol}

We use CLIP-pretrained ViT-B/32, ViT-B/16, and ViT-L/14 as the shared
backbone $\theta_0$. For every task, the zero-shot classification head
constructed from the CLIP text encoder is frozen, and only the image encoder
is fine-tuned from $\theta_0$. This is the standard protocol under which
task vectors $\tau_t=\theta_t^{\star}-\theta_0$ are directly comparable
across tasks and across optimizers.

The AdamW baseline uses learning rate $10^{-5}$ and weight decay $0.1$. For
Muon, all eligible matrix-valued parameters of the image encoder are optimized
with Muon and the remaining vector, scalar, and non-eligible matrix
parameters (biases, normalization parameters, and the input embedding and
output projection layers) are handled by an auxiliary AdamW, following the
standard Muon protocol~\citep{jordan2024muon}. Muon uses learning rate
$10^{-4}$, Nesterov momentum $\beta=0.95$, and five Newton--Schulz
iterations; the auxiliary AdamW uses learning rate $10^{-5}$, so that the
learning rate on the vector parameters is identical to the single-optimizer
AdamW baseline. All runs use cosine learning-rate decay with $500$ warm-up
steps, gradient norm clipped to $1.0$, effective batch size $128$, and random
seed $1993$. The per-task number of fine-tuning steps follows the schedule of
Task Arithmetic~\citep{ilharco2023} and is identical across
optimizers, so the difference between the AdamW and Muon columns of
Table 1 is attributable to the update geometry rather than to
tuning effort. We use this configuration as \emph{Muon} throughout the paper
and do not tune the optimizer assignment separately for individual
Transformer components.

\subsection{Baselines}
\label{app:mm_baselines}

We compare Muon against three pre-merging strategies that span the main
existing axes for improving mergeability. For each strategy, we report both
the AdamW version and the version obtained by replacing AdamW with Muon on
the eligible matrix-valued parameters, denoted by the ``$+$ Muon'' suffix.

\textbf{Non-linear Fine-tuning}~\citep{ilharco2022editing} is the standard
task arithmetic approach, in which each task is fine-tuned in the original
non-linear parameter space starting from $\theta_0$. It serves as the direct
baseline for measuring the effect of the optimizer alone, without any
additional pre-merging mechanism.

\textbf{Tangent Task Arithmetic
(TTA)}~\citep{ortizjimenez2023} fine-tunes in the tangent space around
$\theta_0$, so that the resulting task vectors are constrained to the linear
regime of the NTK. It represents the parameterization-side pre-merging axis,
i.e., improving mergeability by modifying where fine-tuning takes place.

\textbf{OrthoReg}~\citep{liu2026understanding} augments non-linear
fine-tuning with an explicit column-wise orthogonality regularizer on the
weight updates $\Delta W$, i.e., a Frobenius penalty $\|(\Delta W)^{\top}
\Delta W - I\|_F^2$ on every tuned linear layer. It represents the loss-side
pre-merging axis, i.e., improving mergeability by modifying the training
objective.

Pairing each of these three strategies with Muon therefore isolates the
effect of optimizer-level spectral control from the choice of pre-merging
strategy, and lets us test whether the improvement predicted by
Theorem~1 is orthogonal to the existing
loss-side and parameterization-side mechanisms.

\subsection{Merging Coefficient and Metrics}

After fine-tuning, we build task vectors $\tau_t=\theta_t^{\star}-\theta_0$
and compose them via Task Arithmetic ,
\begin{equation}
    \theta_{\mathrm{MT}}
    \;=\;
    \theta_0 + \alpha\sum_{t=1}^{8}\tau_t,
\end{equation}
where a single coefficient $\alpha$ is shared across all tasks and selected
on the validation sets from $\{0,0.05,\ldots,1.0\}$ according to normalized
accuracy. The selected values are $\alpha=0.25$ for ViT-B/32 and ViT-B/16
and $\alpha=0.35$ for ViT-L/14; the same selected $\alpha$ is used across
optimizers so that any performance gap is not confounded with a
better-chosen merging coefficient.

We report two metrics. \emph{Absolute Accuracy} (Abs.Acc.) is the mean
top-$1$ accuracy of the merged model over the eight tasks. \emph{Normalized
Accuracy} (Norm.Acc.) is the mean ratio between the merged-model accuracy on
each task and the corresponding single-task expert accuracy. Norm.Acc.
therefore reflects how much of the single-task capability is preserved after
merging, and factors out the raw difficulty of each individual dataset.In addition, our research also focuses on the accuracy of each unmerged dataset. As shown in Figure \ref{fig:radar_merged} , on different models, Muon has a certain degree of superiority over AdamW.

\begin{figure}[t]
\centering
\includegraphics[width=\linewidth]{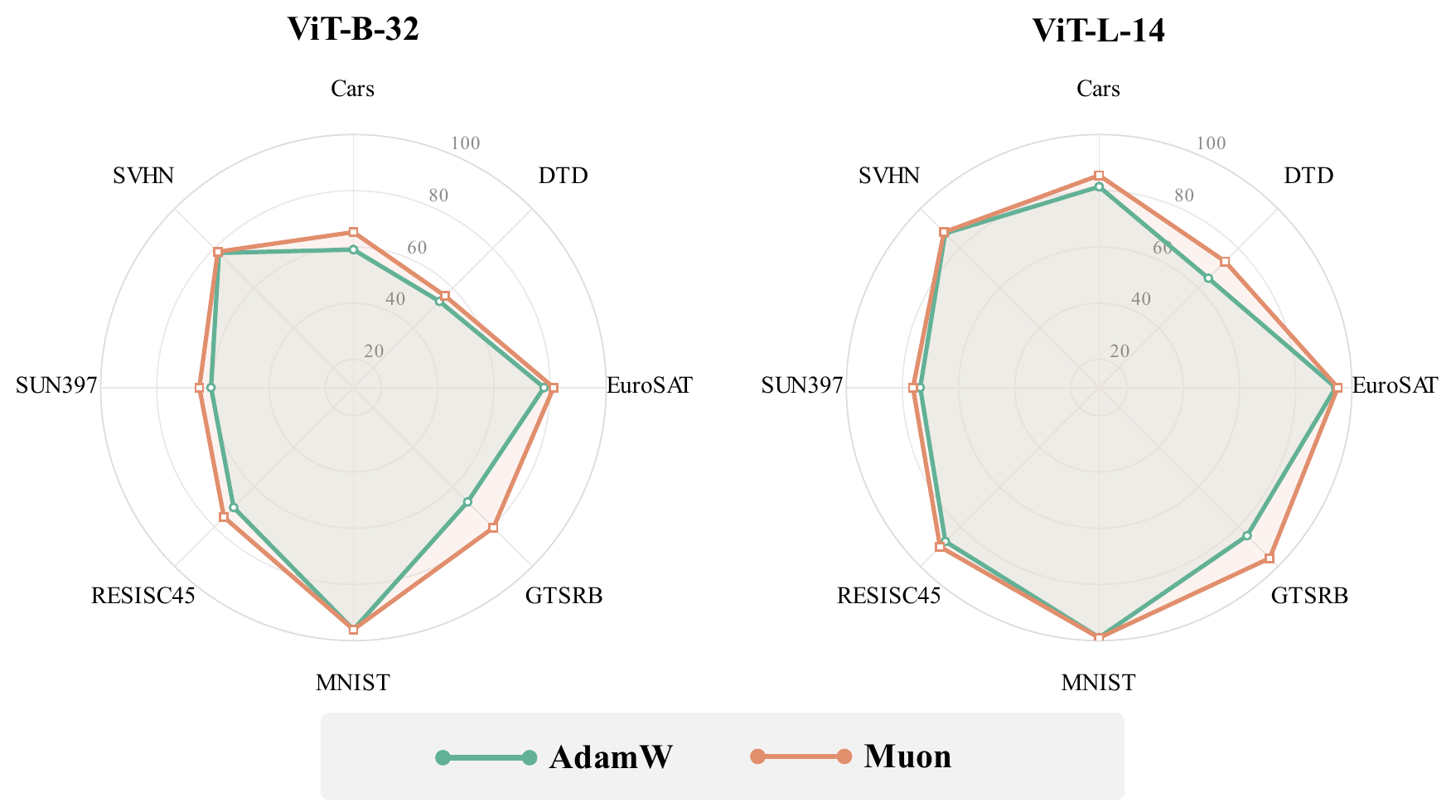}
\caption{Per-task accuracy of the merged model on ViT-B/32 (left) and ViT-L/14 (right).}
\label{fig:radar_merged} 
\end{figure}

\noindent\textbf{Sensitivity to the merging coefficient $\alpha$.}
Figure~\ref{fig:sensitive_alpha} plots Abs.~Acc. of the merged model as a function of $\alpha \in \{0, 0.05, \dots, 1.0\}$. Muon not only attains a higher peak but yields a visibly flatter curve. In addition, our research also focuses on the single-task performance of each unmerged dataset. As shown in Figure ~\ref{fig:sensitive_alpha}, on different models, Muon has a certain degree of superiority over AdamW. 

\begin{figure}[t]
\centering
\includegraphics[width=\linewidth]{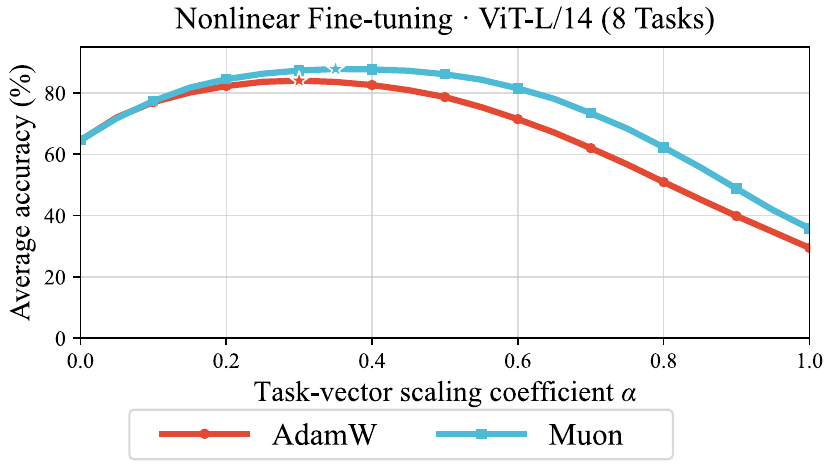}
\caption{Sensitivity of the merged model to the scaling coefficient $\alpha$ on ViT-L/14. Muon reaches a higher peak than AdamW and, more importantly, degrades much more gracefully as $\alpha$ moves away from the optimum.}
\label{fig:sensitive_alpha}
\vspace{-10pt}
\end{figure}

\section{More Details of Continual Learning Experiments}
\label{app:cl}

This appendix expands the CL experimental setup summarized in
Section~5.

\subsection{Choice of Base Methods.}

We evaluate Muon on two baselines that lie at opposite ends of the interference-mitigation spectrum, so that the effect of the optimizer alone can be isolated from architectural mechanisms.

The first is plain LoRA fine-tuning, which inserts low-rank adapters into every attention and MLP projection of the CLIP image encoder and freezes the backbone. It contains no rehearsal buffer, no routing, no regularization, and no gradient projection, so any change in cross-task retention is attributable to the geometry of the update alone. We use it in both class-incremental (LoRA4CIL) and task-incremental (LoRA4TIL) protocols.

The second is MoE-Adapters4CL (Yu et al. 2024a), a widely recognized parameter-efficient framework that already suppresses interference by construction through the DDAS distribution auto-selector, top-2 expert routing, and an incremental activate–freeze strategy. It is one of the strongest baselines on both CIL and MTIL for vision-language models. We choose it for two reasons. First, it delivers strong results on both CIL and MTIL under a single unified architecture. Second, its trainable parameters consist entirely of matrix-valued LoRA experts — exactly the parameters that Muon updates — so switching from AdamW to Muon requires no re-parameterization and the comparison is genuinely controlled at the optimizer level.

For both baselines, "+ Muon" denotes exactly the same architecture, data order, training budget, and (for MoE-Adapters4CL) DDAS checkpoints as the AdamW baseline, with only the optimizer used for the matrix-valued adapter parameters changed.

\subsection{Benchmarks and Protocols}

\textbf{Class-Incremental Learning (CIL).}
We evaluate on three image-classification benchmarks under CIL.
CIFAR-100~(Krizhevsky, Hinton et al.\ 2009) contains 100 classes and
60,000 images, uniformly split into 10, 20, or 50 disjoint tasks with
10, 5, or 2 classes per task, respectively. TinyImageNet~(Le, Yang et
al.\ 2015) contains 200 classes and 100,000 images of size $64\times64$,
organized as a 100-class base task followed by 5, 10, or 20 incremental
steps of 20, 10, or 5 classes each. ImageNet-R~(Hendrycks et al.\ 2021)
is a 200-class subset of ImageNet spanning multiple visual styles; we
adopt the 10-task $\times$ 20-class split to assess CIL under
distribution shift. These three benchmarks jointly stress short and
long task sequences as well as robustness to the style shift
represented by ImageNet-R.

\textbf{Multi-Domain Task-Incremental Learning (MTIL).}
We adopt the full-shot Order-I benchmark of~\citet{ZSCL}
and learn the following 11 datasets sequentially: Aircraft,
Caltech101, CIFAR-100, DTD, EuroSAT, Flowers, Food-101, MNIST,
Oxford-IIIT~Pet, Stanford Cars, and SUN397. The sequence spans
fine-grained, object, texture, remote-sensing, digit, food, and scene
recognition, and therefore stresses cross-task interference across
heterogeneous visual domains. At evaluation time, the Distribution
Discriminative Auto-Selector (DDAS) routes in-distribution inputs to a
learned task router and forwards unseen inputs to the frozen CLIP
model, following MoE-Adapters4CL~(Yu et al.\ 2024a).

\textbf{Metrics.}
For CIL, \emph{Avg.} is the mean cumulative top-1 accuracy over all
learning stages, and \emph{Last} is the cumulative top-1 accuracy after
the final task. For MTIL, we compute the full $T\times T$ accuracy
matrix in which entry $(i,j)$ is the accuracy on dataset $i$ after
learning task $j$. Let $A_{i,j}$ denote this entry. We report:
\begin{align}
\text{Transfer} &= \frac{1}{T-1}\sum_{j=1}^{T-1}\frac{1}{T-j}\sum_{i=j+1}^{T}A_{i,j}, \\
\text{Average} &= \frac{1}{T^2}\sum_{i,j=1}^{T}A_{i,j}, \\
\text{Last}    &= \frac{1}{T}\sum_{i=1}^{T}A_{i,T}.
\end{align}
\emph{Transfer} averages performance on tasks \emph{before} they are
observed, and therefore measures zero-shot generalization.
\emph{Average} is the mean of all matrix entries. \emph{Last} averages
accuracy on all datasets after the final task and therefore emphasizes
retention after the complete task sequence.

\subsection{Architecture and Optimization Protocol}

\textbf{Backbone.}
All CL experiments use a CLIP-pretrained ViT-B/16 image encoder with
the backbone parameters frozen. Zero-shot classification heads are
constructed from the frozen CLIP text encoder, following the standard
MoE-Adapters4CL protocol.

\textbf{MoE-Adapters configuration.}
In CIL, a single router and two experts are shared by all class
increments. In MTIL, we use 22 LoRA experts and one task-specific
router per dataset, with top-2 expert selection and the incremental
activate--freeze strategy of MoE-Adapters4CL~(Yu et al.\ 2024a). Only
the LoRA experts and routers are trainable; the CLIP backbone and text
encoder remain frozen throughout.

\textbf{Optimizer assignment.}
Following the standard Muon protocol~(Jordan et al.\ 2024), every
eligible matrix-valued MoE parameter is optimized with Muon using $K=5$
Newton--Schulz iterations, while vector and scalar parameters (biases,
LayerNorm parameters, router head bias) are handled by an auxiliary
AdamW. Both branches use cosine learning-rate decay and zero weight
decay. The AdamW baseline optimizes the same set of parameters with a
single AdamW instance.

\textbf{DDAS.}
The Distribution Discriminative Auto-Selector uses TinyImageNet as its
reference dataset and a threshold of $0.0655$. To isolate the optimizer
effect on the MoE-Adapter and router parameters, both the AdamW and the
Muon runs share the same Adam-trained DDAS checkpoints; any performance
gap therefore arises from the MoE optimization step rather than from
the task-identification module.

\begin{table}[t]
\centering
\begin{tabular}{lcccccc}
\toprule
 & \multicolumn{2}{c}{5 tasks} & \multicolumn{2}{c}{10 tasks} & \multicolumn{2}{c}{20 tasks} \\
\cmidrule(lr){2-3} \cmidrule(lr){4-5} \cmidrule(lr){6-7}
Method & Avg. & Last & Avg. & Last & Avg. & Last \\
\midrule
\multicolumn{7}{l}{\textit{LoRA4TIL}} \\
AdamW & 93.92 & 91.81 & 91.68 & 90.15 & 91.81 & 90.44 \\
Muon  & \textbf{95.61} & \textbf{93.45} & \textbf{95.07} & \textbf{91.94} & \textbf{96.07} & \textbf{93.64} \\
\bottomrule
\end{tabular}
\caption{CIFAR-100 (5/10/20-task protocols) + Lora for TIL.}
\label{tab:cifar100_til}
\end{table}

\begin{table}[t]
\centering
\begin{tabular}{lcccccc}
\toprule
 & \multicolumn{4}{c}{ImageNet-R} & \multicolumn{2}{c}{TinyImageNet} \\
\cmidrule(lr){2-5} \cmidrule(lr){6-7}
 & \multicolumn{2}{c}{10 tasks} & \multicolumn{2}{c}{20 tasks} & \multicolumn{2}{c}{20 tasks} \\
\cmidrule(lr){2-3} \cmidrule(lr){4-5} \cmidrule(lr){6-7}
Method & Avg. & Last & Avg. & Last & Avg. & Last \\
\midrule
\multicolumn{7}{l}{\textit{LoRA4TIL}} \\
AdamW & 87.97          & 87.87          & 90.33          & \textbf{89.18} & 91.60          & 87.88 \\
Muon  & \textbf{91.17} & \textbf{88.48} & \textbf{91.85} & 88.45          & \textbf{91.75} & \textbf{88.04} \\
\bottomrule
\end{tabular}
\caption{ImageNet-R (10/20-task) and TinyImageNet (20-task) + LoRA for TIL.}
\label{tab:imagenetr_tinyimagenet_til}
\end{table}

\begin{table}[h]
\centering
\begin{tabular}{lcccc}
\toprule
 & \multicolumn{2}{c}{10 tasks} & \multicolumn{2}{c}{50 tasks} \\
\cmidrule(lr){2-3} \cmidrule(lr){4-5}
Method & Avg. & Last & Avg. & Last \\
\midrule
\multicolumn{5}{l}{\textit{LoRA4CIL}} \\
AdamW  & 77.56          & 66.56          & 62.04          & 47.87 \\
Muon   & \textbf{78.73} & \textbf{67.69} & \textbf{64.39} & \textbf{51.01} \\
\bottomrule
\end{tabular}
\caption{CIFAR-100 (10/50-task protocols) + Lora for CIL.}
\label{tab:cifar100_cil}
\end{table}

\begin{table*}[h]\small
\centering
{\setlength{\tabcolsep}{12pt}
\begin{tabular}{lrrrrrr}
\toprule
& \multicolumn{2}{c}{10 tasks} & \multicolumn{2}{c}{20 tasks} & \multicolumn{2}{c}{50 tasks} \\
\cmidrule(lr){2-3}\cmidrule(lr){4-5}\cmidrule(lr){6-7}
Method & Avg. & Last & Avg. & Last & Avg. & Last \\
\midrule
UCIR   \cite{UCIR}                           & 58.66 & 43.39 & 58.17 & 40.63 & 56.86 & 37.09 \\
BiC      \cite{BIC}                         & 68.80 & 53.54 & 66.48 & 47.02 & 62.09 & 41.04 \\
PODNet  \cite{PODNet}                            & 58.03 & 41.05 & 53.97 & 35.02 & 51.19 & 32.99 \\
DER   \cite{DER}                            & 74.64 & 64.35 & 73.98 & 62.55 & 72.05 & 59.76 \\
DyTox+   \cite{DyTox}                      & 74.10 & 62.34 & 71.62 & 57.43 & 68.90 & 51.09 \\
DNE       \cite{DNE}                       & 74.86 & 70.04 & --    & --    & --    & --    \\
\midrule
CLIP                      & 74.47 & 65.92 & 75.20 & 65.74 & 75.67 & 65.94 \\
Fine-tune                           & 65.46 & 53.23 & 59.69 & 43.13 & 39.23 & 18.89 \\
LwF    \cite{Li2018LwF}                            & 65.86 & 48.04 & 60.64 & 40.56 & 47.69 & 32.90 \\
iCaRL     \cite{iCaRL}                        & 79.35 & 70.97 & 73.32 & 64.55 & 71.28 & 59.07 \\
LwF-VR   \cite{LwF-VR}                          & 78.81 & 70.75 & 74.54 & 63.54 & 71.02 & 59.45 \\
ZSCL     \cite{ZSCL}                        & 82.15 & 73.65 & 80.39 & 69.58 & 79.92 & 67.36 \\
\midrule
\multicolumn{2}{l}{\textit{MoE-Adapters4CL}}\\
AdamW    & 85.30 & 77.56 & 84.36 & 76.61 & 83.39 & 73.76 \\
Muon          & \textbf{85.98} & \textbf{78.82} & \textbf{84.87} & \textbf{77.27} & \textbf{83.84} & \textbf{74.02} \\
\bottomrule
\end{tabular}}
\caption{class-incremental CIFAR-100 +  MoE-Adapter for CL.  }
\label{tab:cl_cifar_sota-supp}
\end{table*}

\begin{table}[h]
\centering
\begin{tabular}{lcccc}
\toprule
 & \multicolumn{2}{c}{10 tasks} & \multicolumn{2}{c}{20 tasks} \\
\cmidrule(lr){2-3} \cmidrule(lr){4-5}
Method & Avg. & Last & Avg. & Last \\
\midrule
\multicolumn{5}{l}{\textit{LoRA4CIL}} \\
AdamW & 79.43        & 71.22        & 71.04        & 64.92        \\
Muon  & \textbf{81.55} & \textbf{75.12} & \textbf{75.19} & \textbf{68.43} \\
\bottomrule
\end{tabular}
\caption{ImageNet-R (10/20-task protocols) + Lora for CIL.}
\label{tab:imagenetr_cil}
\end{table}

\subsection{Results}

\paragraph{LoRA4TIL.}
Tables~\ref{tab:cifar100_til} and ~\ref{tab:imagenetr_tinyimagenet_til} report task-incremental
results for the LoRA4TIL baseline. On CIFAR-100, Muon delivers consistent
and substantial improvements over AdamW across all three task-length
protocols: $+1.69$/$+1.64$ (Avg./Last) at 5 tasks, $+3.39$/$+1.79$ at
10 tasks, and $+4.26$/$+3.20$ at 20 tasks. The widening gap with sequence
length suggests that the benefit of spectral step normalization grows as
more tasks compete for the same parameter space. On ImageNet-R, Muon
improves Avg.\ at both the 10- and 20-task settings ($+3.20$ and $+1.52$,
respectively). On TinyImageNet, Muon matches or slightly exceeds AdamW at
the 20-task setting, indicating that orthogonal updates remain effective
even when within-task similarity is high.

\paragraph{LoRA4CIL.}
Under the class-incremental protocol, Muon consistently outperforms AdamW
across benchmarks and task-split configurations. The largest gains appear on
ImageNet-R, where Muon improves Avg.\ by $+2.12$ and $+4.15$ points and Last
by $+3.90$ and $+3.51$ points at the 10- and 20-task settings, respectively.
On CIFAR-100, Muon leads at both the 10- and 50-task settings, with gains of
$+1.17$/$+1.13$ at 10 tasks and $+2.35$/$+3.14$ at 50 tasks; the larger
improvement at 50 tasks suggests that spectral normalization becomes
increasingly beneficial as the number of increments grows. On TinyImageNet
the improvements are small but consistently positive ($+0.19$ to $+1.34$
across all splits).

\begin{table}[t]
\centering
\begin{tabular}{lcccc}
\toprule
 & \multicolumn{2}{c}{10 tasks} & \multicolumn{2}{c}{20 tasks} \\
\cmidrule(lr){2-3} \cmidrule(lr){4-5}
Method & Avg. & Last & Avg. & Last \\
\midrule
\multicolumn{5}{l}{\textit{LoRA4CIL}} \\
AdamW & 74.81        & 64.44        & 67.44        & 56.31        \\
Muon  & \textbf{75.00} & \textbf{64.93} & \textbf{68.72} & \textbf{57.65} \\
\bottomrule
\end{tabular}
\caption{TinyImageNet (10/20-task protocols) + Lora for CIL.}
\label{tab:tinyimagenet_cil}
\end{table}

\paragraph{MoE-Adapters4CL.}
When Muon replaces AdamW inside the MoE-Adapters4CL framework, improvements
are consistent across every benchmark and every task-length protocol.
On CIFAR-100 (Table~\ref{tab:cl_cifar_sota-supp}), Muon improves Avg.\ by
$+0.68$, $+0.51$, and $+0.45$ points and Last by $+1.26$, $+0.66$, and
$+0.26$ points at the 10-, 20-, and 50-task settings, respectively.
On TinyImageNet (Table~\ref{tab:cl_tiny_sota-supp}), Muon achieves
Avg./Last gains of $+0.77$/$+0.53$, $+0.88$/$+0.96$, and $+0.69$/$+0.18$
at the 5-, 10-, and 20-step settings, respectively.
These gains are obtained on top of a framework that already incorporates
top-2 expert routing and the incremental activate--freeze strategy,
demonstrating that the optimizer contributes an orthogonal benefit
independent of the architectural interference-mitigation mechanism.

\begin{table*}[h]\small
\centering
{\setlength{\tabcolsep}{12pt}
\begin{tabular}{lrrrrrr}
\toprule
& \multicolumn{2}{c}{5 steps} & \multicolumn{2}{c}{10 steps} & \multicolumn{2}{c}{20 steps} \\
\cmidrule(lr){2-3}\cmidrule(lr){4-5}\cmidrule(lr){6-7}
Method & Avg. & Last & Avg. & Last & Avg. & Last \\
\midrule
EWC   \cite{EWC}                              & 19.01 &  6.00 & 15.82 &  3.79 & 12.35 &  4.73 \\
EEIL   \cite{EEIL}                            & 47.17 & 35.12 & 45.03 & 34.64 & 40.41 & 29.72 \\
UCIR    \cite{UCIR}                            & 50.30 & 39.42 & 48.58 & 37.29 & 42.84 & 30.85 \\
MUC     \cite{MUC}                            & 32.23 & 19.20 & 26.67 & 15.33 & 21.89 & 10.32 \\
PASS   \cite{PASS}                             & 49.54 & 41.64 & 47.19 & 39.27 & 42.01 & 32.93 \\
DyTox   \cite{DyTox}                            & 55.58 & 47.23 & 52.26 & 42.79 & 46.18 & 36.21 \\
\midrule
CLIP                     & 69.62 & 65.30 & 69.55 & 65.59 & 69.49 & 65.30 \\
Fine-tune                           & 61.54 & 46.66 & 57.05 & 41.54 & 54.62 & 44.55 \\
LwF     \cite{Li2018LwF}                            & 60.97 & 48.77 & 57.60 & 44.00 & 54.79 & 42.26 \\
iCaRL    \cite{iCaRL}                           & 77.02 & 70.39 & 73.48 & 65.97 & 69.65 & 64.68 \\
LwF-VR    \cite{LwF-VR}                          & 77.56 & 70.89 & 74.12 & 67.05 & 69.94 & 63.89 \\
ZSCL     \cite{ZSCL}                           & 80.27 & 73.57 & 78.61 & 71.62 & 77.18 & 68.30 \\
\midrule
\multicolumn{2}{l}{\textit{MoE-Adapters4CL}}\\
AdamW    & 80.85 & 77.04 & 80.26 & 75.79 & 79.98 & 75.68 \\
Muon         & \textbf{81.62} & \textbf{77.57} & \textbf{81.14} & \textbf{76.75} & \textbf{80.67} & \textbf{75.86} \\

\bottomrule
\end{tabular}}
\caption{
TinyImageNet +  MoE-Adapter for CL (5/10/20-step protocols , each with a 100-class base task).}
\label{tab:cl_tiny_sota-supp}
\end{table*}

\begin{table*}[h]\small
\centering
\resizebox{\textwidth}{!}{
\begin{tabular}{llrrrrrrrrrrrr}
\toprule
Metric & Method & Aircraft & Caltech101 & CIFAR-100 & DTD & EuroSAT &
Flowers & Food-101 & MNIST & OxfordPet & Cars & SUN397 & Mean \\
\midrule
\multirow{9}{*}{\emph{Transfer}}
& Continual-FT & -- & 67.1 & 46.0 & 32.1 & 35.6 & 35.0 & 57.7 & 44.1 & 60.8 & 20.5 & 46.6 & 44.6 \\
& LwF \cite{Li2018LwF}          & -- & 74.5 & 56.9 & 39.1 & \textbf{51.1} & 52.6 & 72.8 & 60.6 & 75.1 & 30.3 & 55.9 & 58.9 \\
& iCaRL \cite{iCaRL}       & -- & 56.6 & 44.6 & 32.7 & 39.3 & 46.6 & 68.0 & 46.0 & 77.4 & 31.9 & 60.5 & 50.4 \\
& LwF-VR \cite{LwF-VR}      & -- & 77.1 & \underline{61.0} & 40.5 & 45.3 & 54.4 & 74.6 & 47.9 & 76.7 & 36.3 & 58.6 & 57.2 \\
& WiSE-FT  \cite{WiSE-FT}    & -- & 73.5 & 55.6 & 35.6 & 41.5 & 47.0 & 68.3 & 53.9 & 69.3 & 26.8 & 51.9 & 52.3 \\
& ZSCL   \cite{ZSCL}      & -- & \underline{86.0} & \textbf{67.4} & \textbf{45.4} & \underline{50.4} & \textbf{69.1} & \underline{87.6} & \underline{61.8} & \underline{86.8} & \underline{60.1} & \textbf{66.8} & \textbf{68.1} \\
\cmidrule(lr){2-14}
& \multicolumn{2}{l}{\textit{MoE-Adapters4CL}} \\
& AdamW        & -- & \textbf{88.4} & 53.4 & 43.9 & 39.8 & \textbf{69.1} & \textbf{88.5} & \textbf{62.1} & \textbf{89.1} & \textbf{64.7} & \underline{65.4} & 66.4 \\
& Muon         & -- & \textbf{88.4} & 59.2 & \underline{44.2} & 47.1 & 64.0 & \textbf{88.5} & 60.1 & \textbf{89.1} & \textbf{64.7} & \underline{65.4} & \underline{67.1} \\
\midrule
\multirow{9}{*}{\emph{Average}}
& Continual-FT & 25.5 & 81.5 & 59.1 & 53.2 & 64.7 & 51.8 & 63.2 & 64.3 & 69.7 & 31.8 & 49.7 & 55.9 \\
& LwF \cite{Li2018LwF}         & 36.3 & 86.9 & 72.0 & 59.0 & 73.7 & 60.0 & 73.6 & 74.8 & 80.0 & 37.3 & 58.1 & 64.7 \\
& iCaRL  \cite{iCaRL}      & 35.5 & 89.2 & 72.2 & 60.6 & 68.8 & 70.0 & 78.2 & 62.3 & 81.8 & 41.2 & 62.5 & 65.7 \\
& LwF-VR   \cite{LwF-VR}    & 29.6 & 87.7 & 74.4 & 59.5 & 72.4 & 63.6 & 77.0 & 66.7 & 81.2 & 43.7 & 60.7 & 65.1 \\
& WiSE-FT  \cite{WiSE-FT}    & 26.7 & 86.5 & 64.3 & 57.1 & 65.7 & 58.7 & 71.1 & 70.5 & 75.8 & 36.9 & 54.6 & 60.7 \\
& ZSCL   \cite{ZSCL}      & 45.1 & 92.0 & 80.1 & 64.3 & \underline{79.5} & 81.6 & \textbf{89.6} & \underline{75.2} & \underline{88.9} & 64.7 & \textbf{68.0} & 75.4 \\
\cmidrule(lr){2-14}
& \multicolumn{2}{l}{\textit{MoE-Adapters4CL}} \\
& AdamW        & \underline{52.4} & \textbf{93.3} & \underline{81.6} & \underline{70.0} & 75.3 & \textbf{84.5} & \underline{88.8} & \textbf{75.4} & \textbf{89.1} & \underline{68.5} & \underline{66.8} & \underline{76.9} \\
& Muon         & \textbf{64.2} & \underline{92.9} & \textbf{83.2} & \textbf{70.9} & \textbf{79.8} & \underline{81.9} & 88.7 & 74.4 & \textbf{89.1} & \textbf{69.0} & 66.7 & \textbf{78.3} \\
\midrule
\multirow{9}{*}{\emph{Last}}
& Continual-FT & 31.0 & 89.3 & 65.8 & 67.3 & 88.9 & 71.1 & 85.6 & \textbf{99.6} & 92.9 & 77.3 & 81.1 & 77.3 \\
& LwF \cite{Li2018LwF}          & 26.3 & 87.5 & 71.9 & 66.6 & 79.9 & 66.9 & 83.8 & \textbf{99.6} & 92.1 & 66.1 & 80.4 & 74.6 \\
& iCaRL  \cite{iCaRL}      & 35.8 & \underline{93.0} & 77.0 & 70.2 & 83.3 & 88.5 & \underline{90.4} & 86.7 & \underline{93.2} & 81.2 & \textbf{81.9} & 80.1 \\
& LwF-VR  \cite{LwF-VR}     & 20.5 & 89.8 & 72.3 & 67.6 & 85.5 & 73.8 & 85.7 & \textbf{99.6} & 93.1 & 73.3 & 80.9 & 76.6 \\
& WiSE-FT  \cite{WiSE-FT}    & 27.2 & 90.8 & 68.0 & 68.9 & 86.9 & 74.0 & 87.6 & \textbf{99.6} & 92.6 & 77.8 & \underline{81.3} & 77.7 \\
& ZSCL   \cite{ZSCL}      & 40.6 & 92.2 & 81.3 & 70.5 & 94.8 & 90.5 & \textbf{91.9} & 98.7 & \textbf{93.9} & 85.3 & 80.2 & 83.6 \\
\cmidrule(lr){2-14}
& \multicolumn{2}{l}{\textit{MoE-Adapters4CL}} \\
& AdamW        & \underline{52.4} & \textbf{93.4} & \underline{87.8} & \underline{80.0} & \underline{95.6} & \textbf{97.3} & 89.1 & 98.7 & 89.2 & \underline{85.5} & 79.8 & \underline{86.3} \\
& Muon         & \textbf{64.2} & 92.6 & \textbf{88.6} & \textbf{81.0} & \textbf{98.4} & \underline{96.8} & 89.0 & \underline{99.5} & 89.2 & \textbf{88.2} & 79.2 & \textbf{87.9} \\
\bottomrule
\end{tabular}}
\caption{Full 11-task MTIL evaluation on the benchmark of~\cite{ZSCL}. We report Transfer, Average, and Last. Relative to AdamW, Muon improves the mean Transfer by $+0.6$, Average by $+1.4$, and Last by $+1.6$ points.}
\label{tab:cl_mtil-supp}
\end{table*}
\paragraph{Multi-Domain Task-Incremental Learning.}

Table~\ref{tab:cl_mtil-supp} presents the full 11-task MTIL evaluation.
Relative to the AdamW baseline, Muon improves the mean Transfer score by
$+0.6$ points, the mean Average score by $+1.4$
points, and the mean Last score by $+1.6$ points
. The most striking single-dataset improvement is on
Aircraft, where Muon raises all three metrics by $+11.8$ points ($64.2$
vs.\ $52.4$), a domain requiring fine-grained visual discrimination whose
low-rank gradient structure may particularly benefit from spectral
normalization of the update matrix. Muon also achieves notable gains on
EuroSAT Last ($+2.8$; $98.4$ vs.\ $95.6$) and Stanford Cars Last ($+2.7$;
$88.2$ vs.\ $85.5$). On a few datasets (notably Flowers Average and MNIST
Average), AdamW retains a small advantage, consistent with the
observation that tasks with very high baseline accuracy leave less room for
the geometric regularization provided by Muon.

Taken together, the results across both baselines and all protocols confirm
that Muon provides a systematic benefit in the continual learning setting
that is attributable solely to the optimizer change, since all other
components of the pipeline, including architecture, data order, training
budget, and DDAS checkpoints, are held constant.

\end{document}